\documentclass{article}

\newcommand{\Ex}{{\mathbb{E}}}

\newcommand{\comment}[1]{}
\newcommand{\indi}{\mathbbm{1}}

\newcommand{\mdp}{{\cal M}}
\newcommand{\extM}[1]{\tilde{\mdp}^{#1}}
\newcommand{\todo}[1]{\textcolor{red}{TODO: #1}}

\newcommand{\regStrict}{\mbox{Regret}}
\newcommand{\reg}{{\cal R}}
\newcommand{\st}{{\cal S}}
\newcommand{\ac}{{\cal A}}
\newcommand{\Dir}{\mbox{Dirichlet}}
\newcommand{\smallN}{\eta}
\newcommand{\smallNValue}{\textcolor{black}{\sqrt{\frac{TS}{A}} + 12\omega S^4}} 
\newcommand{\omegaValue}{\textcolor{black}{720\log(n/\rho)}}

\newcommand{\e}{{\boldsymbol{1}}}
\newcommand{\ns}{\psi}
\newcommand{\boost}{\kappa}
\newcommand{\boostValue}{\textcolor{black}{120\log(n/\rho)}}

\newcommand{\optimismBound}{\textcolor{black}{O\left(D \log^2(T/\rho) \sqrt{\frac{SA}{T}}\right)}}
\newcommand{\deviationBound}[1]{O\left(#1(\sqrt{\frac{S}{N^{\tau_k}_{s,a}}} + \frac{ S}{N^{\tau_k}_{s,a}})\log^2(SAT/\rho)\right)}
\newcommand{\deltaValue}{\sqrt{\frac{3\hat{p}_i\log(4S)}{n}}+\frac{3\log(4S)}{n}}
\newcommand{\DeltaValue}[2]{\textcolor{black}{\min\left\{\sqrt{\frac{3 #1 \log(4S)}{#2}} + \frac{3\log(4S)}{#2}, #1\right\}}}

\newenvironment{myproof}{\begin{proof}}{\end{proof} \endproof }

 \usepackage[final]{nips_2017}

\usepackage[utf8]{inputenc} 
\usepackage[T1]{fontenc}    
\usepackage{hyperref}       
\usepackage{url}            
\usepackage{booktabs}       
\usepackage{amsfonts}       
\usepackage{nicefrac}       
\usepackage{microtype}      

\usepackage{amsmath}
\usepackage{amssymb}
\usepackage{amsthm}
\usepackage{graphicx}
\usepackage{url}
\usepackage{color}
\usepackage{algorithm}
\usepackage{algorithmic}
\usepackage{bbm}
\usepackage{times}
\usepackage{thmtools,thm-restate}
\usepackage{enumitem}
\newtheorem{fact}{Fact}

\newtheorem{lemma}{Lemma}[section]
\newtheorem{proposition}[lemma]{Proposition}
\newtheorem{corollary}[lemma]{Corollary}
\newtheorem{theorem}{Theorem}
\newtheorem{definition}{Definition}
\newtheorem{remark}{Remark}

\title{Posterior sampling for reinforcement learning: worst-case regret bounds}

\author{
  Shipra Agrawal\\
  Columbia University\\
  \texttt{sa3305@columbia.edu} \\
   \And
   Randy Jia \\
   Columbia University \\
   \texttt{rqj2000@columbia.edu} \\
}

\begin{document}

\maketitle

\begin{abstract}
We present an algorithm based on posterior sampling (aka Thompson sampling) that achieves near-optimal worst-case regret bounds when the underlying Markov Decision Process (MDP) is communicating with a finite, though unknown, diameter. 
Our main result is a high probability regret upper bound of $\tilde{O}(DS\sqrt{AT})$ for any communicating MDP with $S$ states, $A$ actions and diameter $D$. Here, regret compares the total reward achieved by the algorithm to the total expected reward of an optimal infinite-horizon undiscounted average reward policy, in time horizon $T$. 
This result closely matches the known lower bound of $\Omega(\sqrt{DSAT})$. 
Our techniques involve proving some novel results about the anti-concentration of Dirichlet distribution, which may be of independent interest. 
\end{abstract}

\section{Introduction}

Reinforcement Learning (RL) refers to the problem of learning and planning in sequential decision making systems when the underlying system dynamics are unknown, and may need to be learned by trying out different options and observing their outcomes.  
A typical model for the sequential decision making problem is a Markov Decision Process (MDP), which proceeds in discrete time steps. At each time step, the system is in some state $s$, and the decision maker may take any available action $a$ to obtain a (possibly stochastic) reward. The system then transitions to the next state according to a fixed state transition distribution. 
The reward and the next state depend on the current state $s$ and the action $a$, but are independent of all the previous states and actions. 
In the reinforcement learning problem, the underlying state transition distributions and/or reward distributions are unknown,
and need to be {\em learned} 
using the observed rewards and state transitions, while aiming to {\em maximize} the cumulative reward. This requires the algorithm to manage the tradeoff between  exploration vs. exploitation, i.e., exploring different actions in different states in order to learn the model more accurately vs. taking actions that currently seem to be reward maximizing. 

Exploration-exploitation tradeoff has been studied extensively in the context of stochastic multi-armed bandit (MAB) problems, which are essentially MDPs with a single state. The performance of MAB algorithms is typically measured through {\it regret}, which compares the total reward obtained by the algorithm to the total expected reward of an optimal action. Optimal regret bounds have been established for many variations of MAB (see \cite{bookBubeckCB} for a survey), with a large majority of results obtained using the Upper Confidence Bound (UCB) algorithm, or more generally, the {\it optimism in the face of uncertainty} principle.  Under this principle, the learning algorithm maintains tight over-estimates (or optimistic estimates) of the expected rewards for individual actions, and at any given step, picks the action with the highest optimistic estimate. 
More recently, posterior sampling, aka Thompson Sampling \citep{Thompson}, has emerged as another popular algorithm design principle in MAB, owing its popularity to a simple and extendible algorithmic structure, an attractive empirical performance \citep{ChapelleL11, KaufmannMunos12}, as well as provably optimal performance bounds that have been recently obtained for many variations of MAB \citep{AgrawalG12, agrawal2013further, agrawal-contextual, Russo1, Russo2, BubeckL14}. In this approach, the algorithm maintains a Bayesian posterior distribution for the expected reward of every action; then at any given step, it generates an independent sample from each of these posteriors, and takes the action with the highest sample value. 

In this paper, we consider the Reinforcement Learning (RL) problem in a similar regret based framework, where the total reward of the reinforcement learning algorithm is compared to the total expected reward achieved by a single benchmark policy over a time horizon $T$. In our setting, the benchmark policy is the {\it infinite-horizon undiscounted average reward} optimal policy for the underlying MDP. Here, the underlying MDP is assumed to have finite states $S$ and finite actions $A$, and is assumed be communicating with (unknown) finite diameter $D$. The diameter $D$ is an upper bound on the time it takes to move from any state $s$ to any other state $s'$ using an appropriate policy, for each pair $s,s'$.  A finite diameter is believed to be necessary for interesting bounds on the regret of any algorithm in this setting \citep{jaksch2010near}. 
The UCRL2 algorithm of \citet{jaksch2010near}, which is based on the optimism principle, achieved the first finite regret upper bound of $\tilde{O}(DS\sqrt{AT})$ for this problem. A similar bound was achieved by \citet{bartlett2009regal}, although under known diameter $D$.  \citet{jaksch2010near} also established a worst-case lower bound of $\Omega(\sqrt{DSAT})$ on the regret of any algorithm for this problem. Very recently, an (unpublished) work by \citet{DBLP:journals/corr/abs-1905-12425} has claimed to achieve a regret bound matching the lower bound of $O(\sqrt{DSAT})$ using a variation of the UCRL2 approach. 

{\bf Our main contribution} is a posterior sampling based algorithm with a high probability worst-case regret upper bound of $\tilde{O}(DS\sqrt{AT})$. 
Our algorithm uses an `optimistic' version of the posterior sampling heuristic, while utilizing several ideas from the algorithm design structure in \cite{jaksch2010near}, such as an epoch based execution and the extended MDP construction. The algorithm proceeds in epochs, where in the beginning of every epoch, it generates $\ns=\tilde{O}(S)$ sample transition probability vectors from a posterior distribution for every state and action, and solves an extended MDP with $\psi A$ actions and $S$ states formed using these samples. The optimal policy computed for this extended MDP is used throughout the epoch. 

Posterior Sampling for Reinforcement Learning (PSRL) approach has been studied previously in \cite{osband2013more, abbasi2014bayesian, osband2016posterior}, but in a {\it Bayesian regret} framework. Bayesian regret is defined as the expected regret over a known prior on the transition probability matrix. \cite{osband2016posterior} demonstrate an $\tilde{O}(H\sqrt{SAT})$ bound  on the expected Bayesian regret for PSRL in finite-horizon {\it episodic} Markov decision processes, when the episode length is $H$. In this paper, we consider the stronger notion of {\it worst-case}  regret, aka minimax regret, which requires bounding the maximum regret for any instance of the problem\footnote{Worst-case regret is a strictly stronger notion of regret than Bayesian regret. However, a caveat is that the reward distributions are assumed to be bounded or sub-Gaussian in order to prove worst-case regret bounds. 
On the other hand, the Bayesian regret bounds in the above-mentioned literature allow more general (known) priors on the reward distributions with possibly unbounded support. Bayesian regret bounds under such more general reward distributions are incomparable to the worst-case regret bounds presented here.}.
We consider a {\it non-episodic communicating MDP} setting and prove a worst-case regret bound of $\tilde{O}(DS\sqrt{AT})$, where $D$ is the unknown diameter of the communicating MDP. 
In comparison to a single sample from the posterior in PSRL, our algorithm is slightly inefficient as it uses multiple ($\tilde{O}(S)$) samples. It is not entirely clear if the extra samples are only an artifact of the analysis. In an empirical study of a multiple sample version of posterior sampling for RL, \cite{fonteneau2013optimistic} show that multiple samples can potentially improve the performance of posterior sampling in terms of probability of taking the optimal decision. Our analysis utilizes some ideas from the Bayesian regret analysis. However, bounding the worst-case regret requires several new technical ideas, in particular, for proving `optimism' of the gain of the sampled MDP. Further discussion is provided in Section \ref{sec:regret}.  

PSRL (and our optimistic PSRL) approaches are referred to as ``model-based" approaches, since they explicitly estimate the transition probability matrix underlying the MDP model. Another line of closely related works investigate optimistic versions of ``model-free algorithms" like  of value-iteration \citep{azar2017minimax} and Q-learning \citep{kakade2018}. 
However, the setting considered in both of these works is that of an {\it episodic MDP}, where the learning agent interacts with the system in episodes of fixed and known length $H$. Under this setting, both these works achieve minimax (i.e., worst-case) regret bound of $\tilde{O}(\sqrt{HSAT})$ when $T$ is large enough compared to the episode length $H$.
To understand the challenges in our setting compared to the episodic setting, note that while the initial state of each episode can be arbitrary in the episodic setting, importantly, the sequence of these initial states is shared by the algorithm and any benchmark policy. In contrast, in the non-episodic setting considered in this paper, the state trajectory of the benchmark policy over $T$ time steps can be completely different from the algorithm's trajectory. To the best of our understanding, the shared sequence of initial states of every episode, and the fixed known length $H$ of episodes seem to form crucial components of the analysis in the episodic settings of \cite{azar2017minimax, kakade2018}. Thus, it  would be difficult to extend such an analysis to the non-episodic communicating MDP setting considered in this paper.

Among \noindent {\bf other related work}, \cite{burnetas97} and \cite{tewari2008optimistic} present optimistic linear programming approaches that achieve logarithmic regret bounds with problem dependent constants.  
Strong PAC bounds have been provided in \mbox{\cite{kearns1999finite}, \cite{brafman2003}, \cite{kakade2003sample}}, \cite{Asmuth2009}, \cite{dann2015sample}. There, the aim is to bound the performance of the policy learned at the end of the learning horizon, and not the performance during learning as quantified here by regret. 
{Notably, the BOSS algorithm proposed in \cite{Asmuth2009} is similar to the algorithm proposed here in the sense that the former also takes multiple samples from the posterior to form an extended (referred to as \emph{merged}) MDP.}
\cite{Strehl2005, Strehl2008} provide an optimistic algorithm for bounding regret in a discounted reward setting, but the definition of regret is different in that it measures the difference between the rewards of an optimal policy and the rewards of the learning algorithm {\it  on the state trajectory taken by the learning algorithm}.


\section{Preliminaries and Problem Definition}
\label{sec:prelims}
\subsection{Markov Decision Process (MDP)} 
We consider a Markov Decision Process ${\cal M}$ defined by tuple $\{{\cal S}, {\cal A}, P, r, s_1\}$, where ${\cal S}$ is a finite state-space of size $S$, ${\cal A}$ is a finite action-space of size $A$, $P: \st \times \ac \rightarrow \Delta^\st$ is the transition model, $r: \st \times \ac \rightarrow [0,1]$ is the reward function, and $s_1$ is the starting state. When an action $a\in \ac$ is taken in a state $s\in \st$, a reward $r_{s,a}$ is generated and the system transitions to the next state $s' \in {\cal S}$ with probability $P_{s,a}(s')$, where $\sum_{s'\in {\cal S}} P_{s,a}(s') = 1$. 

We consider `communicating' MDPs with finite `diameter'. Below we define communicating MDPs, and recall some useful known results for such MDPs.  


\begin{definition}[Policy] A deterministic policy $\pi: \st \rightarrow \ac$ is a mapping from state space to action space.
\end{definition}

\begin{definition}[Diameter $D({\cal M})$]
Diameter $D({\cal M})$ of an MDP ${\cal M}$ is defined as the minimum time required to go from one state to another in the MDP using some deterministic policy:
\[ D({\cal M}) = \max_{s \ne s', s,s'\in \st} \min_{\pi: \st \rightarrow \ac} T^{\pi}_{s\rightarrow s'},\]
where $T^{\pi}_{s \rightarrow s'}$ is the expected number of steps it takes to reach state $s'$ when starting from state $s$ and using policy $\pi$.  
\end{definition}

\begin{definition}[Communicating MDP]
\label{def:communicating}
An MDP ${\cal M}$ is communicating if and only if it has a finite diameter. 
That is, for any two states $s\ne s'$, there exists a policy $\pi$ such that the expected number of steps to reach $s'$ from $s$, $T^{\pi}_{s\rightarrow s'}$, is at most $D$, for some finite $D \ge 0$. 
\end{definition}
\begin{definition}[Gain of a policy]
The gain $\lambda^\pi(s)$ of a policy $\pi$, from starting state $s_1=s$, is defined as the infinite horizon undiscounted average reward, given by 
\[\lambda^{\pi}(s) = \Ex[\lim\limits_{T\rightarrow \infty} \frac{1}{T}\sum\limits_{i=1}^T r_{s_t,\pi(s_t)} |s_1=s].\]
where $s_t$ is the state reached at time $t$, on executing policy $\pi$.
\end{definition}
\begin{lemma} [Optimal gain for communicating MDPs]
\label{lem:commBias}
For a communicating MDP $\mdp$ with diameter $D$:
\begin{itemize}
\item[(a)] (\cite{puterman2014markov} Theorem 8.1.2, Theorem 8.3.2) The optimal (maximum) gain $\lambda^*$ is state independent and is achieved by a deterministic stationary policy $\pi^*$, i.e., there exists a deterministic policy $\pi^*$ such that 
\[ \lambda^*:=\max_{s' \in \st} \max_\pi \lambda^{\pi}(s') = \lambda^{\pi^*}(s), \forall s\in \st.\]
Here, $\pi^*$ is referred to as an optimal policy for MDP $\mdp$.
\item[(b)] (\cite{bartlett2009regal}, Theorem 4)
The optimal gain $\lambda^*$ satisfies the following equations,
\begin{equation}
\label{eq:optLP}
\lambda^* = \min_{h\in \mathbb{R}^S} \max_{s,a} r_{s,a} + P_{s,a}^Th - h_s = \max_a r_{s,a} + P_{s,a}^Th^* - h^*_s, \forall s
\end{equation}
where $h^*$, referred to as the bias vector of MDP $\mdp$, satisfies:
\[\max_s h^*_s - \min_s h^*_s \le D.\]
\end{itemize}
\end{lemma}

\comment{
\paragraph{Regret:} 
Regret compares the total reward of the learning agent to the total reward obtained by using optimal policy $\pi^*$ of MDP ${\cal M}$ for all time steps $t=1,\ldots, T$. 
Let $r_t$ be the reward obtained by the learner at time $t$, and $r^*_t$ be the reward obtained at time $t$ on using optimal policy $\pi^*$ for all time steps. Then, the cumulative regret of learning agent is defined as:
\[ \regStrict({\cal M}, T) := \sum_{t=1}^T (r^*_t -r_t)\]
Using standard concentration inequalities, it is easy to show that for communicating MDPs with diameter $D$ (refer to Lemma \ref{lem:rewardToGain} in Section \ref{sec:regret} \todo{add proof}) with probability $1-\rho$, regret can be bounded as
\[ \regStrict({\cal M}, T) \le T\lambda^*({\cal M}) - \sum_{t=1}^T r_t +  O(D\sqrt{T \log(1/\rho)}).\] 
Given above observation, in the rest of the paper, we focus on pseudo-regret:
\[\reg(T) := T\lambda^* - \sum_{t=1}^T r_t \]
to produce the corresponding bound on regret.

Note that in most of the related literature (e.g., \cite{jaksch2010near} and \cite{tewari2008optimistic}) regret is directly defined as the quantity $\reg(T)$ defined above. \todo{check} 
}
Given the above definitions and results, we can now define the reinforcement learning problem studied in this paper.

\subsection{The reinforcement learning problem} 
The reinforcement learning problem proceeds in rounds $t=1,\ldots, T$. The learning agent starts from a state $s_1$ at round $t=1$. In the beginning of every round $t$, the agent takes an action $a_t\in {\cal A}$ and observes the reward $r_{s_t, a_t}$ as well as the next state $s_{t+1} \sim P_{s_t, a_t}$, where $r$ and $P$ are the reward function and the transition model, respectively, for a communicating MDP ${\cal M}$ with diameter $D$. 

The learning agent knows the state-space ${\cal S}$, the action space ${\cal A}$, as well as the rewards $r_{s,a}, \forall s\in \st, a\in \ac$, for the underlying MDP, but not the transition model $P$ or the diameter $D$. (The assumption of known and deterministic rewards has been made here only for simplicity of exposition, since the unknown transition model is the main source of difficulty in this problem. 
Our algorithm and results can be extended to bounded stochastic rewards with unknown distributions using standard Thompson Sampling for MAB, e.g., using the techniques in \cite{agrawal2013further}.)
 
The agent can use the past observations to learn the underlying MDP model and decide future actions. 
The goal is to maximize the total reward $\sum_{t=1}^T r_{s_t, a_t}$, or equivalently, minimize the total regret over a time horizon $T$, defined as
\begin{equation}
\label{eq:regret}
\textstyle \reg(T, \mdp) := T\lambda^* - \sum_{t=1}^T r_{s_t, a_t}
\end{equation}
where $\lambda^*$ is the optimal gain of MDP $\mdp$. 

We present an algorithm for the learning agent with a near-optimal upper bound on the regret $\reg(T, \mdp)$ for any communicating MDP $\mdp$ with diameter $D$, thus bounding the worst-case regret over this class of MDPs.

\section{Algorithm Description}
Our algorithm combines the ideas of Posterior sampling (aka Thompson Sampling) with the extended MDP construction used in \cite{jaksch2010near}. Below we first describe the main components of our algorithm. Our algorithm is then summarized as Algorithm \ref{algo:main}.

\noindent {\it Some notations:} $N^t_{s,a}$ denotes the total number of times the algorithm visited state $s$ and played action $a$ until before time $t$, and $N^t_{s,a}(i)$ denotes the number of time steps among these  $N^t_{s,a}$ steps where the next state was $i$, i.e., the steps where a transition from state $s$ to $i$ was observed. We index the states from $1$ to $S$, so that $\sum_{i=1}^S N^t_{s,a}(i) = N^t_{s,a}$ for any $t$. We use the symbol $\e$ to denote the vector of all $1$s, and $\e_i$ to denote the vector with $1$ at the $i^{th}$ coordinate and $0$ elsewhere. 

\paragraph{Doubling epochs:} Our algorithm uses the epoch based execution framework of \cite{jaksch2010near}. An epoch is a group of consecutive rounds. The rounds $t=1,\ldots, T$ are broken into consecutive epochs as follows: the $k^{th}$ epoch begins at the round $\tau_{k}$ immediately after the end of $(k-1)^{th}$ epoch and ends at the first round $\tau$ such that for some state-action pair $s,a$, $N^{\tau}_{s,a} \ge 2 N^{\tau_{k}}_{s,a}$. The algorithm computes a new policy $\tilde{\pi}_k$ at the beginning of every epoch $k$, and uses that policy through all the rounds in that epoch. Since the total number of visits to any state action-pair is bounded by $T$, it is easy to conclude that irrespective of how the policies $\{\tilde{\pi}_k\}$ are computed, the number of epochs is bounded by $SA\log(T)$.

\paragraph{Posterior Sampling:}  
We use posterior sampling to compute the policy $\tilde{\pi}_k$ in the beginning of every epoch $k$. 
Our algorithm maintains a posterior distribution over the transition probability vector $P_{s,a}$, for every $s\in \st, a\in \ac$. Observe that $P_{s,a}$ specifies a categorical distribution over states ${1, \ldots, S}$, with parameters $P_{s,a}(i), i=1,\ldots, S$.  
Dirichlet distribution is a convenient choice for maintaining a posterior over parameters $P_{s,a}$, as Dirichlet distribution is a conjugate prior for the categorical distribution.
In particular, it satisfies the following useful property: given a prior $\Dir(\alpha_1, \ldots, \alpha_S)$ on $P_{s,a}$, after observing a transition from state $s$ to $i$ (with underlying probability $P_{s,a}(i)$), the posterior distribution is given by $\Dir(\alpha_1, \ldots, \alpha_i+1, \ldots,\alpha_S)$. By this property, for any $s\in \st,a \in \ac$, on starting from prior $\Dir({\bf 1})$ for $P_{s,a}$, the posterior at time $t$ is 
$\Dir(\{N^t_{s,a}(i)+1\}_{i=1, \ldots, S})$. 

 A direct application of the Posterior Sampling for Reinforcement Learning (PSRL) approach introduced in \cite{osband2013more} would involve sampling a transition probability vector from the Dirichlet posterior for each state-action pair, in order to form a sample MDP. A sample policy $\tilde{\pi}_k$ would then be computed as an optimal policy for the sampled MDP.  Our algorithm uses a modified, optimistic version of this approach. At the beginning of every epoch $k$, for every $s\in \st,a\in \ac$ such that $N^{\tau_k}_{s,a} \ge \smallN $, it  generates {\it multiple} samples for $P_{s,a}$ from a {\it boosted} variance posterior. Specifically, for each $s,a$, it generates $\ns$ independent sample probability vectors $Q^{1,k}_{s,a}, \ldots, Q^{\ns,k}_{s,a}$ as 
\[Q^{j,k}_{s,a} \sim \Dir({\bf M}^{\tau_k}_{s,a}), \]
where ${\bf M}^t_{s,a}$ denotes the vector $[M^t_{s,a}(i)]_{i=1, \ldots, S}$, with
\begin{equation}
\label{eq:boostedM}
\textstyle M^t_{s,a}(i) := \frac{1}{\boost} (N^t_{s,a}(i)+\omega), \text{ for } i=1,\ldots, S.
\end{equation}
Here, $\ns, \boost, \omega, \smallN$ are parameters of the algorithm. The values of these parameters are initialized as $\smallN=\smallNValue$, $\ns=\Theta(S\log(SA/\rho))$, $\boost=\Theta(\log(T/\rho))$, $\omega=\Theta(\log(T/\rho))$, for some $\rho\in (0,1]$.  In the regret analysis, we derive sufficiently large constants that can be used in the definition of $\ns, \boost,\omega$ to guarantee the bounds. In particular, in our proofs we use that $\ns = CS\log(SA/\rho)$ with $C = 7^{\frac{32}{\phi}}$, where $\phi = (\frac{(1-\Phi)(\frac{1}{2})}{2})^4$ and $\Phi$ is the normal cumulative distribution function. However, we emphasize that no attempt has been made to optimize those constants, and it is likely that much smaller constants suffice. 

For every remaining $s,a$, i.e., those with small number of visits so far (i.e., those with $N^{\tau_k}_{s,a}<\smallN$) the algorithm uses a simple optimistic sampling described in Algorithm \ref{algo:main}. 
This special sampling has been introduced to handle a technical difficulty in analyzing the anti-concentration of Dirichlet posteriors when the parameters are very small. We suspect that with an improved analysis, this may not be required.


\begin{algorithm}[t] 
\caption{A posterior sampling based algorithm for the reinforcement learning problem
\label{algo:main}
}
  \begin{algorithmic}
	\STATE {\bf Inputs:} State space ${\cal S}$, Action space ${\cal A}$, starting state $s_1$, reward function $r$, time horizon $T$, parameters $\rho \in (0,1], \ns=\Theta(S\log(SA/\rho)), \omega=\Theta(\log(T/\rho)), \boost=\Theta(\log(T/\rho)), \smallN=\smallNValue$. 
	\STATE {\bf Initialize:} $\tau^1:=1$, ${\bf M}^{\tau_1}_{s,a}=\omega {\bf 1}$.
	\vspace{0.08in}
\FORALL{epochs $k=1, 2, \ldots, $} 
\vspace{0.08in}
		\STATE \underline{\it Sample transition probability vectors:} For each $s,a$, generate $\ns$ independent sample probability vectors $Q_{s,a}^{j,k}, j=1,\ldots, \ns$, as follows:
		\begin{itemize} 
		\item {\bf (Posterior sampling)}: For $s,a$ such that $N^{\tau_k}_{s,a}\ge\smallN$, use samples from the Dirichlet distribution:
		\[
		\begin{array}{ll}
		Q^{j,k}_{s,a} \sim \Dir({\bf M}^{\tau_k}_{s,a}), & 
		\end{array}
		\]
		\item {\bf (Simple optimistic sampling)}:
			For $s,a$ such that $N^{\tau_k}_{s,a}<\smallN$, use the following simple optimistic sampling: let 
				\[P^{-}_{s,a} := \hat{P}_{s,a} - {\boldsymbol \Delta},\] 
				where $\hat{P}_{s,a}(i) := \frac{N^{\tau_k}_{s,a}(i)}{N^{\tau_k}_{s,a}}$, and $\Delta_i:=\DeltaValue{\hat{P}_{s,a}(i)}{N^{\tau_k}_{s,a}}$, 
				and let ${\bf z}$ be a random vector picked uniformly at random from $\{\e_1, \ldots, \e_S\}$; set
		\[
		\begin{array}{ll}
		Q^{j,k}_{s,a} = P^{-}_{s,a} + (1-\sum_{i=1}^S P^{-}_{s,a}(i)){\bf z}.& \\
		\end{array}
		\]
		\end{itemize}
		\vspace{0.08in}
	\STATE \underline{\it Compute policy $\tilde{\pi}^k$:} as the optimal gain policy for extended MDP $\extM{k}$ constructed using sample set $\{Q_{s,a}^{j,k}, j=1,\ldots, \ns, s\in \st, a\in \ac\}$. 
 	\vspace{0.08in}
	\STATE \underline{\it Execute policy $\tilde{\pi}^k$:} 
		\FORALL{time steps $t = \tau_k, \tau_k+1, \ldots, $ until \texttt{break epoch}}
		\STATE Play action $a_t = \tilde{\pi}_k(s_t)$.
			\STATE Observe the transition to the next state $s_{t+1}$.
			\STATE Set $N^{t+1}_{s,a}(i), M^{t+1}_{s,a}(i)$ for all $a\in \ac, s,i\in \st$ as defined (refer to Equation \eqref{eq:boostedM}).
			\STATE If $N^{t+1}_{s_t,a_t} \geq 2N^{\tau_{k}}_{s_t,a_t}$, then set $\tau_{k+1}= t+1$ and \texttt{break epoch}.
		\ENDFOR
\ENDFOR
  	\end{algorithmic}
\end{algorithm}

\paragraph{Extended MDP:} The policy $\tilde{\pi}_k$ used in epoch $k$ is computed as the optimal policy of an {\it extended MDP} $\extM{k}$ defined by the sampled transition probability vectors. The construction of this extended MDP is derived from a similar construction in \cite{jaksch2010near}. Given sampled vectors $\{Q^{j,k}_{s,a}, j=1, \ldots, \ns, s\in \st,a \in \ac\}$, we define an extended MDP $\extM{k}$ by extending the original action space as follows: for every $s,a$, create $\ns$ actions for every action $a \in A$, denote by $a^j$ the action corresponding to action $a$ and sample $j$; then, in MDP $\extM{k}$, on taking action $a^j$ in state $s$, reward is $r_{s,a}$ but the transition to next state follows the transition probability vector $Q_{s,a}^{j,k}$. 

Note that the algorithm uses the optimal policy $\tilde{\pi}_k$ of extended MDP $\extM{k}$ to take actions in the action space $\ac$ which is technically different from the action space of MDP $\extM{k}$, where the policy $\tilde{\pi}_k$ is defined. We slightly abuse the notation to say that the algorithm takes action $a_t=\tilde{\pi}(s_t)$ to mean that the algorithm takes action $a_t=a \in \ac$ when $\tilde{\pi}_k(s_t)=a^j$ for some $j$.

Our algorithm is summarized as Algorithm \ref{algo:main}.

\comment{
\begin{remark}[Policy space of original vs. extended MDP]
\label{rem:1}
Note that technically, the extended MDP $\extM{k}$ and the true MDP $\mdp$ have different action space, and therefore different policy space. In the following regret analysis, we slightly abuse the notation to define the gain of policy $\tilde \pi_k$ on the true MDP $\mdp$, and denote it as  $\lambda(\tilde{\pi}_k)$. This is simply defined as the gain of a policy obtained by mapping every action $a^j = \tilde \pi_k(s)$ to the corresponding action $a \in A$ in the action space of MDP $\mdp$. Similarly, `executing  $\tilde \pi_k$ on MDP $\mdp$' will mean executing the policy obtained by the above mapping.
\end{remark}
}



\section{Regret Bounds}
\label{sec:regret}
We prove that with high probability, the regret of Algorithm \ref{algo:main}  is bounded by $\tilde{O}\left(DS\sqrt{AT}\right)$.

\begin{theorem}
\label{th:main}
For any communicating MDP ${\cal M}$ with $S$ states, $A$ actions, and diameter $D$, with probability $1-\rho$,
for $T\ge \Omega\left(SA\log^4(SAT/\rho)\right)$,
the regret of Algorithm \ref{algo:main} is bounded as:
\[ \reg(T, \mdp) \le O\left( DS\sqrt{AT}\log^3(SAT/\rho) + DS^3A^2\log^3(SAT/\rho)\right).\]
For $T \geq \Omega\left(S^4A^3\right)$, this gives a regret bound of:
\[\reg(T, \mdp) \le O\left( DS\sqrt{AT}\log^3(SAT/\rho) \right) .\]
Here $O\left(\cdot \right)$ notation hides only the absolute constants.
\end{theorem}
\begin{myproof}
Here we provide a proof of the above theorem. The proofs of all the referenced lemmas are provided in the subsequent sections.

As defined in Section \ref{sec:prelims}, 
$$  \reg(T, \mdp) = T\lambda^* - \sum_{t=1}^T r_{s_t,a_t}, $$ where $\lambda^*$ is the optimal gain of MDP $\mdp$, $a_t$ is the action taken and $s_t$ is the state reached by the algorithm at time $t$. 
Algorithm \ref{algo:main} proceeds in epochs $k=1,2,\ldots, K$, where $K\le SA\log(T)$. To bound its regret in time $T$, we separately analyze the regret in each epoch $k$, namely,
\begin{equation}
\label{eq:regk}
 \reg_k := (\tau_{k+1}-\tau_k) \lambda^* -  \sum_{t=\tau_{k}}^{\tau_{k+1}-1} r_{s_t, a_t},
\end{equation}
where $\tau_k$ was defined as the starting time step of epoch $k$.
The proof of epoch regret bound has two main components:
\begin{itemize}[leftmargin=0.2in]
\item[(a)] {\bf Optimism:} Recall that in every epoch $k$, the algorithm runs an optimal gain policy for the extended MDP $\extM{k}$. We show that the extended MDP $\extM{k}$ is optimistic, i.e., its optimal gain is (close to) $\lambda^*$ or higher. Specifically, let $\tilde \lambda_k$ be the optimal gain of the extended MDP $\extM{k}$. In Lemma \ref{lem:optimism} (Section \ref{subsec:optimism}), which forms one of the main novel technical components of our proof, we show that with probability $1-\rho$,
\[\textstyle \tilde{\lambda}_k \ge \lambda^* - \optimismBound.\]
Substituting this upper bound on $\lambda^*$ in the expression for $\reg_k$, we obtain the following bound 
on the epoch regret, with probability $1-\rho$:
\begin{equation}
\label{eq:tmp1}
\reg_k \le \textstyle \sum_{t=\tau_{k}}^{\tau_{k+1}-1}  \left(\tilde{\lambda}_k - r_{s_t, a_t}+ \optimismBound\right).
\end{equation}
\item[(b)] {\bf Deviation bounds:}
Next, note that the first term in the above exression is $\tilde{\lambda}_k$, which is the gain of the algorithm's chosen policy $\tilde \pi_k$ on MDP $\extM{k}$ (with transition probability vectors $\tilde P_{s,a}:= Q^{j,k}_{s,a}$ for some $j$); and the second term is the reward obtained on executing the same policy $\tilde \pi_k$, but on the true MDP $\mdp$ (with transition probability vectors $P_{s,a}$). 
We bound the difference  $\sum_t (\tilde{\lambda}_k - r_{s_t,a_t})$ by bounding the deviation $(\tilde{P}_{s,a}-P_{s,a})$ for every $s,a$. 

We use the relation between the gain, the bias vector, and the reward vector of an optimal policy for a communicating MDP, discussed in Section \ref{sec:prelims}. In order to use this relation for MDP $\extM{k}$, we show that this MDP is communicating, by comparing it to the true MDP $\mdp$, which is assumed to be communicating with diameter $D$. Specifically, in Lemma \ref{lem:diameter} (Section \ref{subsec:diameter}), we prove a bound of $2D$ on the diameter of MDP $\extM{k}$ for all $k$ with probability $1-\rho$, when $T\ge \Omega\left(SA\log^4(SAT/\rho)\right)$.

Therefore, we can use the relation between the gain $\tilde \lambda_k$, the bias vector $\tilde h$, and reward vector of optimal policy $\tilde \pi_k$ for communicating MDP $\extM{k}$ given by Lemma \ref{lem:commBias}, part (b). According to this relation, for any state $s$, and action $a=\tilde \pi_k(s)$, $\tilde \lambda_k = r_{s,a}+ \tilde P_{s,a}^T\tilde h -  \tilde h_s$. Since $a_t = \tilde \pi_k(s_t)$, using this relation:
\begin{eqnarray}
\label{eq:tmp2}
\textstyle \sum_{t=\tau_{k}}^{\tau_{k+1}-1}  \left(\tilde{\lambda}_k - r_{s_t, a_t}\right)  
& = & \textstyle \sum_{t=\tau_{k}}^{\tau_{k+1}-1} (\tilde{P}_{s_t, a_t} - \e_{s_t})^T \tilde{h}\nonumber\\
& = & \textstyle \sum_{t=\tau_{k}}^{\tau_{k+1}-1} (\tilde{P}_{s_t, a_t} - P_{s_t, a_t} +  P_{s_t, a_t} - \e_{s_t})^T \tilde{h}.
\end{eqnarray}
In Lemma \ref{lem:deviation} (Section \ref{subsec:deviation}), we prove that with probability $1-\rho$, for all $s,a$, and all $h\in [0,2D]^S$
\begin{equation}
\label{eq:tmpdev}
(\tilde P_{s,a}- P_{s,a} )^Th \le \deviationBound{D}.
\end{equation}
We can use this result to bound first term in \eqref{eq:tmp2}, by observing that $\tilde{h}\in \mathbb{R}^S$, the bias vector of MDP $\extM{k}$ satisfies (refer to Lemma \ref{lem:commBias}),
\begin{center}
$\max_s \tilde{h}_s - \min_s \tilde{h}_s \le D(\extM{k}) \le 2D$,
\end{center}
where the last inequality holds with probability $1-\rho$, as shown in Lemma \ref{lem:diameter} (Section \ref{subsec:diameter}) which proves a bound of $2D$ on the diameter of MDP $\extM{k}$.

For the second term of (\ref{eq:tmp2}) we observe that $\Ex[\e_{s_{t+1}}^T\tilde{h} |\tilde{\pi}_k, \tilde{h}, s_t] = P_{s_t, a_t}^T\tilde h$ and use Azuma-Hoeffding inequality to obtain with probability $1-\rho$, 
\begin{equation}
\label{eq:tmpsecond}
\textstyle\sum_{t=\tau_k}^{\tau_{k+1}-1} (P_{s_t, a_t} - \e_{s_t})^T \tilde{h} \le O(D\sqrt{(\tau_{k+1}-\tau_k) \log(1/\rho)}).
\end{equation}
\end{itemize}

Substituting the bounds from equations \eqref{eq:tmpdev} and \eqref{eq:tmpsecond} into  \eqref{eq:tmp2}, and combining it with \eqref{eq:tmp1} we obtain the following bound on $\reg_k$ with probability $1-3\rho$: 
\begin{eqnarray}
\label{eq:regEpoch}
\textstyle \reg_k & = &  \textstyle O\left( \left(D (\tau_{k+1}-\tau_k)\sqrt{\frac{SA}{T}}+ D \sum_{s,a} (N^{\tau_{k+1}}_{s,a} - N^{\tau_k}_{s,a})(\frac{\sqrt{S}}{\sqrt{N^{\tau_k}_{s,a}}} + \frac{S}{N^{\tau_k}_{s,a}})  \right) \log^2(\frac{SAT}{\rho}) \right)\nonumber\\
& & \textstyle + O\left(D \sqrt{(\tau_{k+1}-\tau_k)\log(\frac{1}{\rho})}\right).
\end{eqnarray}

We can finish the proof by observing that (by definition of an epoch) the number of visits of any state-action pair can at most double in an epoch, 
\[N^{\tau_{k+1}}_{s,a} - N^{\tau_k}_{s,a} \le N^{\tau_k}_{s,a}. \]
Substituting this observation along with $\sum_k \tau_{k+1}-\tau_k\le T$ and $\sum_{k=1}^K \sqrt{\tau_{k+1}-\tau_k} \le \sqrt{KT}$ in \eqref{eq:regEpoch}, we can bound the total regret $\reg(T) = \sum_{k=1}^K \reg_k $ as the following, with probability $1-3K\rho$,
\begin{eqnarray*}
\sum_{k=1}^K \reg_k & \le & \textstyle O\left(\sum\limits_{k=1}^K \left(D (\tau_{k+1}-\tau_k)\sqrt{\frac{SA}{T}}  + D\sqrt{S}\sum\limits_{s,a} \sqrt{N^{\tau_k}_{s,a}} + DS^2A\right) \log^2(\frac{SAT}{\rho}) + D \sqrt{(\tau_{k+1}-\tau_k)\log(1/\rho)}\right) \\
\hspace{-0.2in} & \le &  \textstyle  O\left(\left( D\sqrt{SAT} +D\sqrt{S}\log(K) (\sum_{s,a} \sqrt{N^{\tau_K}_{s,a}}) + K DS^2A\right) \log^2(\frac{SAT}{\rho}) + D\sqrt{KT\log(\frac{1}{\rho})}\right)  
\end{eqnarray*}
where we used $N^{\tau_{k+1}}_{s,a} \le 2N^{\tau_k}_{s,a}$ and $\sum_k (\tau_{k+1}-\tau_k) = T$. 
Now, we use that $K\le SA\log(T)$, and since $\sum_{s,a} N^{\tau_K}_{s,a} \le T$, by simple worst scenario analysis, $\sum_{s,a} \sqrt{N^{\tau_K}_{s,a}} \le \sqrt{SAT}$, and we obtain,
\[\reg(T, \mdp) \le O\left( DS\sqrt{AT}\log^3(\frac{SAT}{\rho})  +  DS^3A^2\log^3(\frac{SAT}{\rho})  \right) .\]
For $T \geq \Omega(S^4A^3)$, this gives a regret bound of:
\[\reg(T, \mdp) \le O\left( DS\sqrt{AT}\log^3(SAT/\rho) \right) .\]

\end{myproof}
\section{Proofs of the lemmas used in Section \ref{sec:regret}}
\subsection{Notation} 
\label{subsec:notation}
We use the following notations repeatedly in this section. Fix an epoch $k$, state $s$, action $a$, and sample $j$. The specific values of $k,j, s,a$ will be clear from the context in a given proof. We denote $n=N_{s,a}^{\tau_k}$, $n_i=N_{s,a}^{\tau_k}(i)$ for all $i\in {\cal S}$, and $m=\frac{n+\omega S}{\boost}$. Here $\omega=\omegaValue$ and $\boost=\boostValue$, as defined in the algorithm. Also, we denote $p_i=P_{s,a}(i)$, $\hat{p}_i := \frac{n_i}{n}$, $\bar{p}_i = \frac{n_i+\omega}{n+\omega S}$, and $\tilde{p}_i=Q^{j,k}_{s,a}(i)$, for $i\in \st$.

When $n>\smallN$, the algorithm uses Dirichlet posterior sampling to generate sample vectors $Q^{j,k}_{s,a}$, so that in this case $\tilde{p}$ is a random vector distributed as $\Dir(m\bar p_1, \ldots, m \bar p_S)$. 

When $n<\smallN$, simple optimistic sampling is used, so that $\tilde{p}$ was generated as follows: denote 
$$\textstyle p^- = [\hat p - (\deltaValue) \e]^+,$$ 
and let ${\bf z}$ be a random vector picked uniformly at random from $\{\e_1, \ldots, \e_S\}$; then
		\[
		\begin{array}{ll}
		\tilde{p} = p^- + (1-\sum_j p^-_j){\bf z}.& \\
		\end{array}
		\]
We define 
$$\delta_i := \hat{p}_i - p_i, \ \Delta_i := \hat{p}_i - p^{-}_i  = \textstyle \DeltaValue{\hat p_i}{n}.$$
 Note that by multiplicative Chernoff bounds (Fact \ref{mcher}), with probability $1-\frac{1}{2S}$, $|\delta_i| \leq \deltaValue$. Therefore, 
$$\textstyle \sum_i\delta_i =0, \sum _i \Delta_i = \sum_i (\hat{p}_i - p^{-}_i) = 1- \sum_i p^{-}_i, \text{ and } \Delta_i \geq \delta_i \text{ (with probability $1-\frac{1}{2S}$)} $$ 
Above notations and observations will be used repeatedly in the proofs in this section.

\subsection{Optimism}
\label{subsec:optimism}

The goal of this section is to show optimism, i.e.:
\[\textstyle \tilde{\lambda}_k \ge \lambda^* - \tilde O(D\sqrt{\frac{SA}{T}}).\]
First, in Lemma \ref{lem:optimism1} below we prove for any fixed vector, for every $s,a$, there exists a sample transition probability vector whose projection on that vector is optimistic, with high probability. To prove this, we prove the following fundamental new result on the anti-concentration of any fixed projection of a Dirichlet random vector 

\begin{restatable}{proposition}{antiDir}
\label{prop:antiMain}
Fix any vector $h\in \mathbb{R}^S$ such that $|h_i-h_{i'}| \le D$ for any $i,i'$. Consider a random vector $\tilde p$ generated from Dirichlet distribution with parameters $(m\bar p_1, \ldots, m \bar p_S)$, where $m\bar p_i \ge 6$. Then, for any $\rho\in (0,1)$, with probability $\Omega(1/S) - S\rho$,
\[(\tilde{p}-\bar p)^Th \ge  \frac{1}{8}  \sqrt{\sum_{i<S} \frac{\bar \gamma_i \bar c_i^2}{m}} -\frac{2SD\log(2/\rho)}{m}\]
where 
$\bar \gamma_i := \frac{\bar p_i (\bar p_{i+1} +\ldots + \bar p_S)}{(\bar p_i + \ldots + \bar p_S)}, \bar c_i = (h_i - \bar H_{i+1}), \bar H_{i+1} = \frac{1}{\sum_{j=i+1}^S \bar{p}_j}\sum_{j=i+1}^S h_j \bar{p}_j.$
\end{restatable}
The proof is provided in the appendix. 
In the appendix, we also prove the following strong concentration bound for the empirical probability vectors. 
\begin{restatable}{proposition}{concEmp}
\label{prop:ConcentrationFixedh}
Fix any vector $h\in \mathbb{R}^S$ such that $|h_i-h_{i'}| \le D$ for any $i,i'$.
Let $\hat{p}\in \Delta^S$ be the average $n$ independent multinoulli trials with parameter $p \in \Delta^S$, where $n\geq 96$. Then,  for any $\rho\in (0,1)$, with probability $1-\rho$, 
\[|(\hat{p}-p)^Th| \le 2\sqrt{\log(n/\rho) \sum_{i<S} \frac{\gamma_i c_i^2}{n} } + 3D\frac{\log(2/\rho)}{n},\]
where $\gamma_i=\frac{p_i(p_{i+1} + \cdots + p_S)}{(p_i+ \cdots +p_S)}$, $c_i= h_i-H_{i+1}$, $H_{i+1}=\frac{1}{\sum_{j=i+1}^S {p}_j}\sum_{j=i+1}^S h_j {p}_j$. 
\end{restatable}
Together the above two results allow us to prove the following lemma.

\begin{lemma}
\label{lem:optimism1}
Fix any vector $h\in \mathbb{R}^S$ such that $|h_i-h_{i'}| \le D$ for any $i,i'$, and any epoch $k$. Then, for every $s,a$, with probability $1-\frac{\rho}{SA}$ there exists at least one $j$ such that 
\[\textstyle (Q^{j,k}_{s,a})^Th \ge P_{s,a}^Th - \optimismBound.\]
\end{lemma}
\begin{myproof}
Fix an epoch $k$, state and action pair $s,a$, sample $j$. We use the notation defined in Section \ref{subsec:notation}, so that $\tilde p=Q^{j,k}_{s,a}$, $p=P_{s,a}$, etc.
We show that with probability $\Omega(1/S-S\rho)$, $\tilde p^Th \ge p^Th - \optimismBound$. Since we have $\psi  = CS\log(SA/\rho)$ independent samples for every $s,a$, for some large enough constant $C \geq 7^{\frac{32}{\phi}}$ (where $\phi = (\frac{(1-\Phi)(\frac{1}{2})}{2})^4$ and $\Phi$ is the normal cumulative distribution function), this result will give us the lemma statement. 
To prove this result, we consider two cases:\newline\\
\noindent {\bf Case 1: $n>\smallN$.} 
When $n>\smallN$, Dirichlet posterior sampling is used so that $\tilde{p}$ is a random vector distributed as $\Dir(m\bar p_1, \ldots, m \bar p_S)$, where $m=\frac{n+\omega S}{\boost}$, $\bar{p}_i = \frac{n_i+\omega}{n+\omega S}$.  We show that with probability $\Omega(1/S) - 8S\rho$, the random quantity $\tilde{p}^Th$ exceeds its mean $\bar{p}^Th$ enough to overcome the possible deviation of empirical estimate $\bar{p}^Th$  from the true value $p^Th$. This involves combining the Dirichlet anti-concentration bound from Proposition \ref{prop:antiMain} to lower bound $\tilde{p}^Th$ (note that $m\bar p_i \geq \frac{\omega}{\boost} = 6$), and the concentration bound on empirical estimates $\hat{p}$ from Proposition \ref{prop:ConcentrationFixedh} to lower bound $\bar{p}^Th$ (note that $n \geq \eta \geq 96$), which by definition is close to $\hat{p}^Th$. 

In Proposition \ref{prop:antiMain2} (in the appendix), we prove a slight  modification of Proposition \ref{prop:antiMain} to show that with probability $\Omega(1/S) -7S\rho$,
\begin{equation}
\label{eq:antiMain2}
 (\tilde{p}-\bar p)^Th  \ge 0.184 \sqrt{\boost \sum_i \frac{\gamma_i c_i^2}{n}} - O(\frac{DS\omega \log(n/\rho)}{n}).
\end{equation}
Above bound replaces $\bar \gamma_i, \bar c_i, m$ in the lower bound provided by Proposition \ref{prop:antiMain} by $\gamma_i, c_i, n$ instead. With this modification, the lower bound becomes directly comparable to the bound on the deviation $|(\hat p-p)^Th|$ provided by Proposition \ref{prop:ConcentrationFixedh}. 
To combine this lower bound with the deviation bound, we calculate 
\[|(\bar{p} - \hat{p})^Th| = |\sum_{i=1}^S h_i(\frac{n\hat{p}_i + \omega}{n + \omega S} - \frac{n\hat{p}_i}{n})| = |\sum_i h_i(\frac{\omega(1-S\hat{p}_i)}{n + \omega S})| \leq \frac{\omega DS}{n + \omega S} \leq \frac{\omega DS}{n}.\]
Then, using the above bound along with \eqref{eq:antiMain2}, and the result from Proposition \ref{prop:ConcentrationFixedh}, 
we have that with probability $\Omega(1/S)-8S\rho$,
\begin{eqnarray*}
(\tilde p - p)^Th &=& (\tilde{p}-\bar p)^Th + (\bar{p} - \hat{p})^Th+(\hat{p}-p)^Th \\
& \geq & (\tilde{p}-\bar p)^Th - |(\bar{p} - \hat{p})^Th| - |(\hat{p}-p)^Th| \\
& \geq & 0.184 \sqrt{\boost \sum_i \frac{\gamma_i c_i^2}{n}} - 2\sqrt{\log(n/\rho) \sum_{i<S} \frac{\gamma_i c_i^2}{n}} - O(\frac{DS\omega\log(n/\rho)}{n})  \\
& \geq & -O(\omega\frac{DS\log(n/\rho)}{n}) \\
& \ge & -\optimismBound
\end{eqnarray*}
where the second last inequality follows from the observation that with $\boost = \boostValue$, the first term is bigger than the second. Then, substituting $\omega = \omegaValue$ and $n\ge \smallN=\smallNValue$, we obtain the last inequality.\newline\\

\noindent {\bf Case 2: $n<\smallN$.} When $n<\smallN$, simple optimistic sampling is used. Using notation and observations made in Section \ref{subsec:notation}, in this case $\tilde{p} = p^- + (1-\sum_j p^-_j){\bf z}$. 
With probability $1/S$, $z = \e_i$ for an $i$ such that $h_i = \|h\|_\infty$, and (by union bound over all $i$) with probability $1-S\frac{1}{2S} = \frac{1}{2}$, $|\delta_i| \leq \deltaValue$ for every $i$. So with probability at least $1/2S$:
\begin{eqnarray*}
\sum_i \tilde{p}_ih_i &=& \sum_i p^-_ih_i + \|h\|_\infty(1-\sum_j p^-_j)  = \sum_i p^-_i h_i + \|h\|_\infty\sum_j \Delta_j \\
&=& \sum_i (\hat p_i -\Delta_i) h_i + \|h\|_\infty\Delta_i = \sum_i \hat p_ih_i + (\|h\|_\infty-h_i)\Delta_i \\
&\geq& \sum_i \hat p_ih_i + (\|h\|_\infty-h_i)\delta_i = \sum_i (\hat p_i-\delta_i)h_i +\|h\|_\infty \delta_i \\
&=& \sum_i p_ih_i + \|h\|_\infty\sum_i \delta_i = \sum_i p_ih_i. 
\end{eqnarray*}
\end{myproof}

Finally, we use the above lemma to prove the main optimism lemma (Lemma \ref{lem:optimism}).

\begin{lemma}[Optimism]
\label{lem:optimism} 
With probability $1-\rho$, for every epoch $k$, the optimal gain $\tilde \lambda_k$ of the extended MDP $\extM{k}$  satisfies:
\[  \textstyle \tilde \lambda_k \ge \lambda^* -  \optimismBound,\]
 where $\lambda^*$ the optimal gain of MDP $\mdp$ and $D$ is the diameter.
\end{lemma}
\begin{myproof}
Let $h^*$ be the bias vector for an optimal policy $\pi^*$ of MDP $\mdp$ (refer to Lemma \ref{lem:commBias} in the preliminaries section). Since $h^*$ is a fixed (though unknown) vector with $|h_i-h_j|\le D$, we can apply Lemma \ref{lem:optimism1} to obtain that with probability $1-\rho$, for all $s,a$, there exists a sample vector $Q^{j,k}_{s,a}$ for some $j\in \{1,\ldots, \ns\}$ such that
\[(Q^{j,k}_{s,a})^Th^* \ge P_{s,a}^Th^* - \delta\]
where $\delta=\optimismBound$.
Now, consider the policy $\pi$ for MDP $\extM{k}$ which for any $s$, takes action $a^j$, where $a=\pi^*(s)$, and $j$ is a sample satisfying above inequality. Note that $\pi$ is essentially $\pi^*$ but with a different transition probability model. Let $Q_{\pi}$ be the transition matrix for this policy, whose rows are formed by the vectors $Q^{j,k}_{s,\pi^*(s)}$, and $P_{\pi^*}$ be the transition matrix whose rows are formed by the vectors $P_{s,\pi^*(s)}$. Above implies
\[Q_{\pi} h^* \ge P_{\pi^*} h^* - \delta \e.\]

Let ${Q}^*_\pi$ denote the limiting matrix for Markov chain with transition matrix $Q_\pi$. 
Observe that $Q_\pi$ is aperiodic, recurrent and irreducible : it is aperiodic and irreducible because each entry of
$Q_\pi$ being a sample from Dirichlet distribution is non-zero, and it is positive recurrent because in a finite irreducible Markov chain, all states are positive and recurrent. This implies that $Q^*_\pi$ is of the form $\e{\bf q^*}^T$ where ${\bf q^*}$ is the stationary distribution of $Q_\pi$, and $\e$ is the vector of all 1s (refer to (A.6) in \cite{puterman2014markov}). Also, $Q_\pi^* Q_\pi = Q_\pi$, and $Q_\pi^*\e = \e$.

Therefore, the gain of policy $\pi$
\[\tilde\lambda(\pi) \e= (r_{\pi}^T {\bf q}^*) \e = Q^*_\pi r_\pi\]
where $r_{\pi}$ is the $S$ dimensional vector $[r_{s,{\pi}(s)}]_{s=1,\ldots, S}$. Now, 
\[
\begin{array}{rcll}
\tilde{\lambda}(\pi) \e - \lambda^* \e & = & Q_\pi^*r_{\pi} - \lambda^* \e &\\
& = & Q_\pi^*r_{\pi} - \lambda^* (Q_\pi^* \e) & \ldots(\text{using } Q_\pi^*\e=\e)\\
& = & Q_\pi^*(r_\pi - \lambda^* \e) &\\
& = & Q_\pi^*(I-P_{\pi^*})h^* & \ldots(\text{using } \eqref{eq:optLP})\\
& = &  Q_\pi^*(Q_\pi-P_{\pi^*})h^* & \ldots(\text{using } Q_\pi^* Q_\pi=Q_\pi^*)\\
& \ge & -\delta \e & \ldots(\text{using } (Q_\pi-P_{\pi^*})h^*\ge -\delta \e, Q_\pi^*\e=\e).
\end{array}
\]
Then, by optimality, 
\[\tilde{\lambda}_k \ge \tilde\lambda(\pi) \ge \lambda^*-\delta.\]
\end{myproof}


\subsection{Deviation Bounds}
\label{subsec:deviation}


\begin{lemma}
\label{lem:deviation} 
In every epoch $k$, with probability $1-\rho$,  for all samples $j$, all $s,a$, and all vectors $h\in [0, H]^S$,
\[
 \textstyle (Q^{j,k}_{s,a} - P_{s,a})^T h \le 
\deviationBound{H}.
\]
\end{lemma}
\begin{myproof}
Fix an $s,a, j,k$. Let $\tilde{p}=Q_{s,a}^{j,k}$. Denote $n=N^{\tau_k}_{s,a}$ and $m=\frac{n+\omega S}{\kappa}$, and $n_i=N^{\tau_k}_{s,a}(i)$, $\bar{p}_i:=\frac{n_i+\omega}{n+\omega S}$ and $\hat{p}_i:=\frac{n_i}{n}$ for  $i=1,\ldots, S$. Recall that $\smallN =\smallNValue$ and $\omega =\omegaValue$. It suffices to prove the lemma statement for $H=1$. We consider two cases. 

\noindent {\bf Case 1:} When $n >\smallN$, posterior sampling is used. Therefore, $\tilde p$ is an $S$-dimensional Dirichlet random vector with parameters $m\bar{p}_i, i=1,\ldots, S$. 
Let $X$ be distributed as Gaussian with mean $\mu=\bar p^Th$ and variance $\sigma^2=\frac{1}{m}$. Now, for any \emph{fixed} $h\in [0,1]^S$, by Gaussian-Dirichlet stochastic optimism (see Lemma \ref{gvd} in the appendix) 
$$ X \succeq_{so}  \tilde p^Th.$$
Then by Gaussian concentration (Corollary \ref{cor:soconc1}), for any $\rho'\in (0,1)$, and fixed $h\in [0,1]^S$, with probability at least $1-\rho'$, 
\begin{equation}
\label{eq:dev1}
 |\tilde p^Th-\bar p^Th|\le \sqrt{\frac{2}{m}\log(\frac{2}{\rho'})} \le \sqrt{\frac{140}{n}\log(\frac{n}{\rho})\log(\frac{2}{\rho'})} .
\end{equation}
where in the last inequality, we substituted $m\ge \frac{n}{\boost}$, with $\boost=\boostValue$.
In Proposition \ref{prop:ConcentrationFixedh}, we proved a strong bound on $|\hat p^Th-p^Th|$ for any fixed $h\in [0,1]^S$, which was used for proving optimism. A corollary of that concentration bound (by using observations that $\gamma_i=\frac{p_i(p_{i+1}+\cdots+p_S)}{(p_i+\cdots+p_S)}\le p_i$, and $|c_i|\le 1$ when $h\in [0,1]^S$) is that for any $\rho'\in (0,1)$, and fixed $h\in[0,1]^S$ with probability $1-\rho'$,
\begin{equation}
\label{eq:dev2}
|(\hat{p}-p)^Th| \le 2\sqrt{\frac{\log(n/\rho') }{n} } + \frac{3\log(2/\rho')}{n}.
\end{equation}
Also, for {\it all} $h\in [0,1]^S$
\begin{equation}
\label{eq:dev3}
|\hat p^Th - \bar p^Th| \le \sum_i |\frac{n_i+\omega}{n+\omega S} -  \frac{n_i}{n}| \le \frac{\omega S n_i}{(n+\omega S)n} \le \frac{\omega S}{n}.
\end{equation}
where $\omega=\omegaValue$.
Combine the bounds from equation \eqref{eq:dev1}, \eqref{eq:dev2}, and \eqref{eq:dev3}, and take union bound over all fixed $h$ on an $\epsilon$-grid over $[0,1]^S$, with $\epsilon=1/n$. Then, substituting $\rho'$ by $\rho'/n^S$, we have that with probability $1- \rho'$ 
\begin{equation}
\label{eq:dev4}
|\tilde p^Th-p^Th|\le 14\sqrt{\frac{S\log(n/\rho')\log(n/\rho)}{n}}	+  5\frac{S\log(n/\rho')}{n} + \frac{\omega S}{n}.
\end{equation}

\noindent {\bf Case 2:} When $n \leq \smallN$, simple optimistic sampling is used. 
 Using notation in Section \ref{subsec:notation}, in this case $\tilde p = p^{-} + (1-\sum_{i=1}^S p^{-}_i){\bf z}$, 
where ${\bf z}$ be a random vector picked uniformly at random from $\{\e_1, \ldots, \e_S\}$. 
By multiplicative Chernoff bounds (Fact \ref{mcher}) to bound $(\hat{p}-p)$, we have for any $\rho''\in (0,1)$, with probability $1-\rho''$, for all $h \in [0,1]^S$
\begin{eqnarray}
\label{eq:dev5}
(\tilde p^T h - p^Th)  & \le & (\hat p^T h - p^T h) + \sum_i\sqrt{\frac{3\hat{p}_i\log(4S)}{n}} + \sum_i\frac{3\log(4S)}{n} \nonumber\\
&  \leq & ||\hat p - p||_1||h||_\infty + \sqrt{S\frac{3\log(4S)}{n}}+ \frac{3S\log(4S)}{n}\nonumber\\
&  \leq & \sqrt{\frac{2S\log(1/\rho'')}{n}} + \sqrt{S\frac{3\log(4S)}{n}}+ \frac{3S\log(4S)}{n}\nonumber\\
& = & 3\sqrt{\frac{S\log(S/\rho'')}{n}} + 3\frac{S\log(S)}{n}.
\end{eqnarray}
Equations \eqref{eq:dev4} and \eqref{eq:dev5} provide a bound on $|(Q_{s,a}^{j,k})^T\tilde h-P_{s,a}^T\tilde h|$ for any given $s,a,j,k$. Substituting $\rho'=\rho''=\rho/(SA\ns)$,  and taking a union bound over all possible values of $s,a,j$ we get the lemma statement.  (Here $\ns=\Theta(S\log(SA/\rho))$.)
\end{myproof}

\subsection{Diameter of the extended MDP}
\label{subsec:diameter}

Algorithm \ref{algo:main} computes policy $\tilde{\pi}_k$ in epoch $k$ as an optimal gain policy of the extended MDP $\extM{k}$. Our goal in this section is to prove that the diameter of $\extM{k}$ is within a constant factor of the diameter $\mdp$. We begin by deriving a bound on the diameter of $\extM{k}$ under certain conditions, and then prove that those conditions hold with high probability.

\begin{lemma}
\label{lem:diameter2}
Let $E^s \in \mathbb{R}_+^S$ be the vector of the minimum expected times to reach $s$ from $s'\in\st$ in true MDP $\mdp$, i.e., $E^s_{s'} = \min_\pi T^\pi_{s'\rightarrow s}$. Note that $E^s_s=0$. For any episode $k$, if for every $s,a$ there exists some $j$ such that 
 \begin{equation}
\label{eq:1}
Q_{s,a}^{j,k} \cdot E^s \le P_{s,a} \cdot E^s + \delta,
\end{equation}
for some $\delta\in [0,1)$, then the diameter of extended MDP $\extM{k}$ is at most $\frac{D}{1-\delta}$, where $D$ is the diameter of MDP $\mdp$. 
\end{lemma}
\begin{myproof}
Fix an epoch $k$. For brevity, we omit the superscript $k$ in below. 

Fix any two states $s_1 \ne s_2$. We prove the lemma statement by constructing a policy $\tilde{\pi}$ for $\extM{}$ such that the expected time to reach $s_2$ from $s_1$ is at most $\frac{D}{1-\delta}$. Let $\pi$ be the policy for MDP $\mdp$ for which the expected time to reach $s_2$ from $s_1$ is at most $D$ (since $\mdp$ has diameter $D$, such a policy exists). 
Let $E$ be the $|S|-1$ dimensional vector of expected times to reach $s_2$ from every state, except $s_2$ itself, using $\pi$ ($E$ is the sub-vector formed by removing $s_2^{th}$ coordinate of vector $E^{s_2}$ where $E^s$ was defined in the lemma statement. Note that $E^{s_2}_{s_2}=0$). By first step analysis, $E$ is a solution of:
\[E=\e+P^\dagger_\pi E,\] 
where $P^\dagger_{\pi}$ is defined as the $(S-1)\times (S-1)$ transition matrix for policy $\pi$, with the $(s,s')^{th}$ entry being the transition probability $P_{s,\pi(s)}(s')$  for all $s, s'\ne s_2$. Also, by choice of $\pi$, $E$ satisfies
\[E_{s_1} \leq D.\]
Now, we define $\tilde \pi$ using $\pi$ as follows: For any state $s\ne s_2$, let $a=\pi(s)$ and $j^{th}$ sample satisfies the property \eqref{eq:1} for $s,a, E^{s_2}$, then we define $\tilde{\pi}(s):= a^j$. Let $Q_{\tilde \pi}$ be the transition matrix (dimension $S\times S$) for this policy. 

$Q_{\tilde \pi}$ defines a Markov chain. Next, we modify this Markov chain to construct an absorbing Markov chain with a single absorbing state $s_2$. Let ${Q}^\dagger_{\tilde \pi}$ be the submatrix $(S-1)\times (S-1)$ submatrix of $Q_{\tilde \pi}$ obtained by removing the row and column corresponding to the state $s_2$. Then $Q'$ is defined as (an appropriate reordering of) the following matrix:
\[ Q'_{\tilde \pi} =
\begin{bmatrix}
Q^\dagger_{\tilde \pi} & {\bf q} \\
\bold{0} & 1
\end{bmatrix}
\]
where ${\bf q}$ is an $(S-1)$-length vector such that the rows of $Q'_{\tilde \pi}$ sum to $1$. Since the probabilities in $Q_{\tilde \pi}$ were drawn from Dirichlet distribution, they are all strictly greater than $0$ and less than $1$. Therefore each row-sum of $ Q^\dagger_{\tilde \pi}$ is strictly less than $1$, so that the vector ${\bf q}$ has no zero entries and the Markov chain is indeed an absorbing chain with single absorbing state $s_2$. Then we notice that $(I- Q^\dagger_{\tilde \pi})^{-1}$ is precisely the fundamental matrix of this absorbing Markov chain and hence exists and is non-negative (see \cite{grinstead2012introduction}, Theorem 11.4). Let $\tilde E$ be defined as the $S-1$ dimensional vector of  expected time to reach $s_2$ from $s'\ne s_2$ in MDP $\extM{k}$ using $\tilde \pi$. Then, it is same as the expected time to reach the absorbing state $s_2$ from $s'\ne s_2$ in the Markov chain $Q'_{\tilde \pi}$, given by
\[ \tilde{E} = (I-\bar Q^\dagger_{\tilde \pi})^{-1} \e.\]

Then using \eqref{eq:1} (since $E^{s_2}_{s_2}=0$, the inequality holds for $P^\dagger, Q^\dagger$),
\begin{equation}
\label{eq:2}
E = \bold{1} + P^\dagger_{\pi} {E} \ge \bold{1} + Q^\dagger_{\tilde \pi} E - \delta \bold{1} \ \Rightarrow \ \ (I-Q^\dagger_{\tilde \pi}) E \ge (1-\delta) \bold{1}.
\end{equation}


Multiplying the non-negative matrix $(I-Q^\dagger_{\tilde \pi})^{-1}$  on both sides of this inequality, it follows that 
\[E \geq (1-\delta) (I-Q^\dagger_{\tilde \pi})^{-1} \bold{1} = (1-\delta) \tilde{E} \] 
so that $  \tilde{E}_{s_1} \leq \frac{1}{(1-\delta)} E_{s_1} \leq \frac{D}{1-\delta}$, proving that the expected time to reach $s_2$ from $s_1$ using policy $\tilde{\pi}$ in MDP $\extM{k}$ is at most $\frac{D}{1-\delta}$. 

\end{myproof}

\comment{
We can then make a more general statement about any two MDPs (that differ slightly in transitions) that will be useful in the proof of a later result:
\begin{corollary}
\label{cor:diameter2}
Let MDP $\mdp$ have transitions $P_{s,a}$ and optimal policy $\pi$. Let $E^s$  for $\mdp$ under $\pi$ be defined as in Lemma \ref{lem:diameter2}. Now consider another MDP $\mdp'$ with same states, actions, and rewards, but different transition matrix $P'_{s,a}:$ its values drawn from some distribution so that each entry is in $(0,1)$.

If for every $s,a$, $P'_{s,a} \cdot E \leq P_{s,a} \cdot E + \delta$ for some $\delta \in [0,1)$, then \[D(\mdp', \pi) \leq \frac{D(\mdp,\pi)}{1-\delta}.\]
\end{corollary}
\begin{myproof}
Follows directly from the proof of Lemma \ref{lem:diameter2} noting that $\tilde \pi(s):= a^j$ in that proof is just simply $\pi(s)$ here since there is only one sample. In that proof we showed that under this policy, the expected time to reach a state from any other is at most $\frac{D(\mdp,\pi)}{1-\delta}$.
\end{myproof}
}

Now we can use the above result to prove that the diameter of the extended MDP is bounded by twice the diameter of the original MDP:
\begin{lemma}
\label{lem:diameter}
Assume $T\ge CSA\log^4(SAT/\rho)$ for a large enough constant $C$. Then, 
for any epoch $k$, the diameter of MDP $\extM{k}$ is bounded by $2D$, with probability $1-\rho$.
\end{lemma}
\begin{myproof}
Fix an epoch $k$. For any state $s$, let $E^s$ be as defined in Lemma \ref{lem:diameter2}. We show that with probability $1-\rho$, for all $s, a$, there exists some $j$ with $Q^{j,k}_{s,a}\cdot E^s \le P_{s,a} \cdot s + \delta$, with $\delta \le 1/2$. This will allow us to apply Lemma \ref{lem:diameter2} to bound the diameter of $\extM{k}$.

Given any $s,a,j,k$, we use notations and observations from Section \ref{subsec:notation}, so that $\tilde p = Q^{j,k}_{s,a}, p  = P_{s,a}$ etc. Also, let $h=E^s$. Then, $\min_i h_i=0, \max_i h_i = D$.

First consider all $s,a$ with $n > \smallN$.  Using \eqref{eq:dev4} (in the proof of Lemma \ref{lem:deviation}), we have
$$ \tilde p^T h - p^T h \le  14 D\sqrt{\frac{S\log(n/\rho')\log(n/\rho)}{n}}	+  5 D\frac{S\log(n/\rho')}{n} + D \frac{\omega S}{n}, $$
with probability $1-\rho'$ for any $\rho'\in (0,1)$. 
Substituting $\rho'=\rho/(2SA\ns)$, we get that with probability $1-\frac{\rho}{2}$, for all $s,a,j$ such that $n>\smallN$, $\tilde p^T h - p^T h\le  \delta$, where  $\delta = 14\sqrt{\frac{2\log^2(SAT/\rho)}{\smallN}}	+  5\frac{2S\log(SAT/\rho')}{\smallN} + \frac{\omega S}{\smallN}$.
Then, using  $\smallN=\smallNValue$, and $T\ge CSA\log^4(SAT/\rho)$ (for some constant $C$), we get $\delta\le 1/2$. While no attempt has been made to optimize constants, we note that $C \geq 28^4$ is sufficient. 

For $s,a$ such that $n \le \smallN$, simple optimistic sampling is used. Using notations introduced in Section \ref{subsec:notation}, in this case $\tilde p = p^- + (1-\sum_j p^-_j){\bf z}$, where ${\bf z}$ is a random vector picked uniformly at random from $\{\e_1, \ldots, \e_S\}$.
With probability $1/S$, $z = \e_i$ for $i$ such that $h_i = \min_i h_i =0$. Therefore, with probability at least $1/2S$:
\begin{eqnarray*}
\tilde p ^Th &=& (p^-)^Th = \sum_i (\hat p_i -\Delta_i) h_i  \leq \sum_i (\hat p_i-\delta_i)h_i = p^Th.
\end{eqnarray*}
Since we have $\psi  = CS\log(SA/\rho)$ independent samples for every $s,a$, for some large enough constant $C \geq 7^{\frac{32}{\phi}}$ (where $\phi = (\frac{(1-\Phi)(\frac{1}{2})}{2})^4$ and $\Phi$ is the normal cumulative distribution function), with probability $1-\frac{\rho}{2}$, there exists at least one sample $j$ such that $Q^{j,k}_{s,a}\cdot h \ge P_{s,a} \cdot h$. 

Therefore, we have shown that with probability $1-\rho$, for all $s,a$, there exists some $j$ such that $Q^{j,k}_{s,a} \cdot E^s \le P_{s,a}\cdot E^s +\delta$, with $\delta \le 1/2$. By Lemma \ref{lem:diameter2} we obtain that the diameter of $\extM{k}$ is bounded by $D/(1-\delta) \le 2D$ with probability $1-\rho$.
\end{myproof}



\section{Conclusions}
We presented an algorithm inspired by posterior sampling that achieves near-optimal worst-case regret bounds for the reinforcement learning problem with communicating MDPs in a non-episodic, undiscounted average reward setting. 
Our algorithm may be viewed as a randomized version of the UCRL2 algorithm of \cite{jaksch2010near}, with randomization via posterior sampling. Our analysis demonstrates that posterior sampling provides the right amount of uncertainty in the samples, so that an optimistic policy can be obtained without excess over-estimation. 

While our work surmounts some important technical difficulties in obtaining worst-case regret bounds for posterior sampling based algorithms for communicating MDPs, the provided bound matches the previous best bound in $S$ and $A$. Obtaining a better worst-case regret bound remains an open question. In particular, we believe that studying value functions may improve the dependence on $S$ in the regret bound, possibly for large $T$ (\cite{azar2017minimax} produce an $\tilde{O}(\sqrt{HSAT})$ bound when $T\ge H^3S^3 A$). Other important directions of future work include reducing the number of posterior samples required in every epoch from $\tilde{O}(S)$ to constant or logarithmic in $S$, and extensions to contextual and continuous state MDPs.

\bibliographystyle{plainnat}
\bibliography{references}
\newpage

\appendix

\section{Missing proofs from Section \ref{subsec:optimism}}

\subsection{Anti-concentration of Dirichlet distribution: Proof of Proposition \ref{prop:antiMain}}
\label{app:DirAnti}


We prove the following general result on anti-concentration of Dirichlet distributions, which will be used to prove optimism.

\antiDir*
We use an equivalent representation of a Dirichlet vector in terms of independent Beta random variables.
\begin{fact}
\label{fact:DirRepresent}
Fix an ordering of indices $1,\ldots, S$, and define $\tilde{y}_i:=\frac{\tilde p_i}{\tilde p_{i}+\cdots + \tilde p_S},  \bar{y}_i:=\frac{\bar p_i}{\bar p_{i}+\cdots + \bar p_S}$. Then,
for any $h\in \mathbb{R}^S$,
\[(\tilde{p}-\bar p)^Th = \sum_i (\tilde y_i - \bar{y}_i)({h}_i-\tilde{H}_{i+1})(\bar{p}_i+\cdots + \bar p_S) =\sum_i (\tilde y_i - \bar{y}_i)({h}_i-\bar{H}_{i+1})(\tilde{p}_i+\cdots + \tilde p_S)\]
where 
$\tilde{H}_{i+1}= \frac{1}{\sum_{j=i+1}^S \tilde{p}_j}\sum_{j=i+1}^S h_j \tilde{p}_j$, $\bar{H}_{i+1}= \frac{1}{\sum_{j=i+1}^S \bar{p}_j}\sum_{j=i+1}^S h_j \bar{p}_j$.
\end{fact}

\begin{fact}
\label{fact:y}
For $i=1,\ldots, S$, $\tilde{y}_i:=\frac{\tilde p_i}{\tilde p_{i}+\cdots + \tilde p_S}$ are independent Beta random variables distributed as $\mbox{Beta}(m \bar{p}_i, m (\bar p_{{i+1}}+\cdots + \bar p_S))$, with mean 
\[\Ex[\tilde{y}_i]=\frac{m \bar{p}_i}{m (\bar p_{{i}}+\cdots + \bar p_S)} = \bar{y}_i,\]
and variance
\[\bar\sigma_i^2 := \Ex[(\tilde{y}_i - \bar y_i)^2]=\frac{\bar{p}_i (\bar p_{{i+1}}+\cdots + \bar p_S)}{(\bar p_{{i}}+\cdots + \bar p_S)^2(m (\bar p_{{i}}+\cdots + \bar p_S) + 1 )} .\]
\end{fact}


\begin{lemma}[Corollary of Lemma \ref{lem:BetaAntiConcentration}]
\label{lem:yAnti}
Let $\tilde{y_i}, \bar{y}_i, \bar\sigma_i$ be defined as in Fact \ref{fact:y}. If $m\bar{p}_i, m (\bar p_{{i+1}}+\cdots + \bar p_S) \geq 6,$ then, for any positive constant $ C \leq \frac{1}{2}$, 
\[P(\tilde{y_i} \geq \bar{y_i} + C \bar \sigma_i + \frac{C}{m(\bar{p}_i + ... + \bar{p}_S)}) \geq 0.15 =: \eta.\]
\end{lemma}
\begin{myproof}
Apply  Lemma \ref{lem:BetaAntiConcentration} with $a=m\bar{p}_i, b=m (\bar p_{{i+1}}+\cdots + \bar p_S)$.
\end{myproof}


\begin{lemma}(Application of Berry-Esseen theorem)
\label{lem:BEapp}
Let $G \subseteq \{1, \ldots, S\}$ be a set of indices, $z_i \in \mathbb{R}$ be fixed. Let
\[X_G := \sum_{i\in G} (\tilde y_i - \bar{y}_i)z_i.\]
Let $F$ be the cumulative distribution function of 
\[\frac{X_G}{\sigma_G}, \text{ where, } \sigma_G^2=\sum_{i\in G} z_i^2 \bar\sigma_i^2,\]
$\bar \sigma_i$ being the standard deviation of $\tilde y_i$ (refer to Fact \ref{fact:y}). Let $\Phi$ be the cumulative distribution function of standard normal distribution. Then, for all $\epsilon>0$:
\[\sup_x |F(x) - \Phi(x)| \le \epsilon\]
as long as 
\[\sqrt{|G|} \ge \frac{R C}{\epsilon} \text{, where  } R:= \max_{i,j \in G} \frac{z_i\bar\sigma_i}{z_j \bar\sigma_j}\] 
for some $C \le 3+\frac{6}{m \bar p_i}$. 
\end{lemma}
\begin{myproof}
$Y_i=(\tilde y_i - \bar{y}_i)z_i$. Then, $Y_i, i\in G$ are independent variables, with $\Ex[Y_i]=0$, 
\begin{eqnarray*}
\sigma_i^2 := \Ex[Y_i^2] & = & \Ex[(\tilde y_i - \bar{y}_i)^2(z_i)^2] \\
& = & z_i^2 \bar\sigma_i^2
\end{eqnarray*}
\begin{eqnarray*}
\rho_i := \Ex[|Y_i|^3] & \le & \Ex[|Y_i|^4]^{3/4}\\
& = & \Ex[|\tilde y - \bar y|^4]^{3/4} z_i^3\\
& \le & \kappa   \Ex[|\tilde y - \bar y|^2]^{3/2} z_i^3\\
& = & \kappa   \bar\sigma_i^{3} z_i^3 
\end{eqnarray*}
where the first inequality is by using Jensen's inequality and $\kappa$ is the Kurtosis of Beta distribution. Next, we use that $\tilde y$ is Beta distributed, and Kurtosis of $Beta(\nu\mu, \nu(1-\mu))$ Distribution is 
\[\kappa = 3 + \frac{6}{(3+\nu)} \left(\frac{(1-2\mu)^2(1+\nu)}{\mu(1-\mu)(2+\nu)}-1\right) \le 3 + \frac{6}{(3+\nu)\mu}. \]
Here, $\alpha=m(\bar p_i+\cdots + \bar p_S)\bar{y_i}, \beta=m(\bar p_i+\cdots + \bar p_S)(1-\bar{y_i})$, so that
\[\kappa \le 3 + \frac{6}{3+m(\bar p_i+\cdots + \bar p_S)} \frac{1}{\bar y_i} \le 3+\frac{6}{m \bar p_i}.\]

Now, we use Berry-Esseen theorem (Fact \ref{be}), with
\begin{eqnarray*}
\psi_1 & = & \frac{1}{\sqrt{\sum_{i\in G} \sigma_i^2}} \max_{{i\in G}}\frac{\rho_i}{\sigma_i^2} \\
& \le & \frac{\kappa}{\sqrt{|G|}} \frac{\max_{i\in G}  {z_i \bar\sigma_i}}{\min_{i\in G}  {z_i \bar\sigma_i }}
\end{eqnarray*}
to obtain the lemma statement.

\end{myproof}


\begin{lemma}
\label{lem:Anti1}
Assuming $m\bar p_i \ge 6, \forall i$, for any fixed $z_i$, $i=1,\ldots, S$, 
\[\Pr\left( \sum_i (\tilde y_i - \bar{y}_i)z_i \ge \frac{1}{4} \sqrt{\sum_i \bar\sigma_i^2z_i^2}\right) \ge \Omega(1/S).\]
\end{lemma}
\begin{myproof}
Define constant 
$\delta := \frac{(1-\Phi)(\frac{1}{2})}{2}$ and $k(\delta):=\frac{C^2}{\delta^4}$, where $C\leq 4$.

Consider the the group of indices with the $k(\delta)$ largest values of $|z_i \bar\sigma_i|$, call it group $G(1)$, and then divide the remaining into smallest possible collection ${\cal G}$ of groups such that $|z_i\bar\sigma_i|/|z_j\bar\sigma_j| \le \frac{1}{\delta}$ for all $i,j$ in any given group $G$. Define an ordering $\prec$ on groups by ordering them by maximum value of $|z_i\bar \sigma_i|$ in the group.  That is $G \succ G'$ if $\max_{i\in G} z_i^2 \bar\sigma_i^2 \ge \max_{j\in G'} z_j^2\bar\sigma_j^2 $
Note that by construction, for $G \succ G'$, we have $\max_{i\in G} z_i^2 \bar\sigma_i^2 \ge \frac{1}{\delta^2} \max_{j\in G'} z_j^2\bar\sigma_j^2  $. 

Recall from Lemma \ref{lem:BEapp}, for every group  $G \in {\cal G}$ of size $\sqrt{|G|} > \frac{C}{\delta\epsilon}$, we have that its cdf is within $\epsilon$ of normal distribution cdf, giving that $\Pr(X_G \ge \frac{1}{2}\sigma_G) \ge  2\delta -\epsilon$.  Using this result for $\epsilon=\delta$, we get that for every group of size at least $k(\delta)$, we have
 
\begin{equation}
\label{eq:beApp}
\Pr(X_G \ge \frac{1}{2}\sigma_G) \ge  \delta.
\end{equation}

We will look at three types of the groups we created above: 
\begin{itemize}
\item Top big groups: those among the top $\log_{1/\delta}(S)$ groups that have cardinality at least $k(\delta)$
\item Top small groups: those among the top $\log_{1/\delta}(S)$ groups that have cardinality smaller than $k(\delta)$
\item Bottom groups: those not among the top $\log_{1/\delta}(S)$ groups 
\end{itemize}
Here, top groups refers to the those ranked higher according to the ordering $\succ$.

For the first group type above, apply \eqref{eq:beApp} to obtain, 
\begin{center}
$ \text{ for all big groups among top $\log_{1/\delta}(S)$, } X_G \ge \frac{1}{2}\sigma_G $
\end{center}
\begin{equation}
\label{prob:1}
\text{ with probability at least }\delta^{\log_{1/\delta}(S)} = \frac{1}{S}.
\end{equation}
Next, we analyze the remaining indices (among top small groups and bottom groups). Consider the group $G(1)$ we set aside. Using Lemma \ref{lem:yAnti} $k(\delta)$ times, we have:
\[\Pr\left(\sum_{i\in G(1)} (\tilde{y}_{i} - \bar y_{i})z_i \ge 0.5  \sqrt{\sum_{i\in G(1)} z_i^2\bar\sigma_i^2} \right) \ge \eta^{k(\delta)}\]  
where $\eta \ge 0.15$. 

Now, if it is the case where the top group is of small size, we apply the above anticoncentration of beta for each element in the group, so that for all indices $i$ in this group, $(\tilde{y}_{i} - \bar y_{i})z_i \ge 0.5z_i\bar\sigma_i$, with probability $\eta^{k(\delta)}$. 
To conclude, so far, we have with probability at least $\frac{1}{S} \eta^{2k(\delta)}$
\[ \sum_{i\in G(1), i\in \text{top big groups}} (\tilde{y}_{i} - \bar y_{i})z_i \ge 0.5  \sqrt{\sum_{i\in G(1),i \in \text{top big groups}} z_i^2\bar\sigma_i^2}.\]

For every other small group $G$, the group's total variance is at most $ k(\delta) \max_{i\in G} z_i^2 \bar \sigma_i^2 \le k(\delta) \delta^{2j} z_{(1)}^2\bar \sigma_{(1)}^2$, where $j$ is the rank of the group in ordering $\succ$ and $(1)$ is the index of the smallest variance in $G(1)$.
So, the sum of the standard deviation for top $\log_{1/\delta}(S)$ small groups is at most 
\[ k(\delta) \sum_{G:\text{top small groups}} \max_{i\in G} z_i^2 \bar \sigma_i^2  \le k(\delta) \sum_{j=1}^{\log_{1/\delta}(S)} \delta^{2j} z_{(1)}\bar \sigma_{(1)} \le \frac{k(\delta)\delta^2}{1-\delta^2} z_{(1)}^2\bar \sigma_{(1)}^2 \]

as it is a geometric series with $\delta$ multiplier. For the remaining bottom group, each element's variance is at most $\frac{1}{S^2} z_{(1)}^2 \bar \sigma^2_{(1)}$, therefore 
\[\sum_{i:\text{top small groups, bottom groups}}z_i^2\bar\sigma^2_i \leq (\frac{k(\delta)\delta^2}{1-\delta^2}+\frac{1}{S}) z_{(1)}^2 \bar \sigma^2_{(1)} \leq \frac{k(\delta)}{25} z_{(1)}^2 \bar \sigma^2_{(1)} \leq \frac{1}{25} \sum_{i \in G(1)} z_{i}^2 \bar \sigma^2_{i}.\] 

By Cantelli's Inequality (Fact \ref{cant}), with probability at least $\frac{1}{2}$,
\[\sum_{i:\text{top small groups, bottom groups}}  (\tilde{y}_i - \bar y_i)z_i \geq - \sqrt{\sum_{i \in \text{top small groups, bottom groups}} z_{i}^2 \bar \sigma^2_{i}} \geq - \frac{1}{5}\sqrt{\sum_{i \in \text{G(1)}} z_{i}^2 \bar \sigma^2_{i}}.\]
Hence combining our results above,
\begin{eqnarray*}
\sum_{i} (\tilde{y}_i - \bar y_i)z_i  &\geq& \frac{1}{2} \sqrt{\sum_{i \in G(1),\text{top big groups}} z_i^2 \bar\sigma_i^2} - \frac{1}{5} \sqrt{\sum_{i \in G(1)} z_i^2 \bar\sigma_i^2} \\
&\geq& \frac{3}{10} \sqrt{\sum_{i \in G(1),\text{top big groups}} z_i^2 \bar\sigma_i^2} +\frac{1}{25} \sqrt{\sum_{i \in G(1)} z_i^2 \bar\sigma_i^2}  -\frac{1}{25} \sqrt{\sum_{i \in G(1)} z_i^2 \bar\sigma_i^2}  \\
&\geq& \frac{13}{50} \sqrt{\sum_{i \in G(1),\text{top big groups}} z_i^2 \bar\sigma_i^2} +\frac{1}{25} \sqrt{\sum_{i \in G(1)} z_i^2 \bar\sigma_i^2}  \\
&\geq& \frac{1}{4} \sqrt{\sum_{i} z_i^2 \bar\sigma_i^2}  
\end{eqnarray*}
with probability  $\eta^{2k(\delta)}\frac{1}{2S}= \Omega(1/S)$.

\end{myproof}


\begin{myproof} {\bf (Proof of Proposition \ref{prop:antiMain})}
Since $\tilde p$ and $\bar p$ are probability vectors (sum to $1$), it is sufficient to consider $h\in [0,D]^S$. Now, use Fact \ref{fact:DirRepresent} to express $(\tilde{p}-\bar p)^Th$ as:
\[(\tilde{p}-\bar p)^Th = \sum_i (\tilde y_i - \bar{y}_i)({h}_i-\tilde{H}_{i+1})(\bar{p}_i+\cdots + \bar p_S).\]
We note that $\tilde{H}_i$ is the scalar product of an $S-i+1$-dimensional Dirichlet random vector with the fixed vector $(h_i,...,h_S)$, and $\bar H_i$ is the mean of that product. Therefore, we can apply the same argument used in the proof of Case 1 of Lemma \ref{lem:deviation} in Section \ref{subsec:deviation}.   

Let $X$ be distributed as Gaussian with mean $\mu=\bar H_i$ and variance $\sigma^2=\frac{1}{m(\bar p_i + \ldots + \bar p_S)}$. Now, for any \emph{fixed} $h\in [0,D]^S$, by Gaussian-Dirichlet stochastic optimism (Lemma \ref{gvd}), $ X \succeq_{so}  \tilde{H}_i.$
Then by Gaussian concentration (Corollary \ref{cor:soconc1}),
 \[|\tilde{H}_i - \bar H_i| \le D\sqrt{\frac{2\log(2/\rho)}{m(\bar p_i + \ldots + \bar p_S)}}\]
 with probability $1-\rho$ for any $i$.


Similarly, noting that $\tilde y_i$ is a Beta random variable, using Gaussian-Beta stochastic optimism (Lemma \ref{gvb}), if $X$ is distributed as Gaussian with mean $\mu=\bar y_i$ and variance $\sigma^2=\frac{1}{m(\bar p_i + \ldots + \bar p_S)}$, then $ X \succeq_{so}  \tilde{y}_i$. Then by Corollary \ref{cor:soconc1}, with probability $1-\rho$,
\[|\tilde y_i - \bar{y}_i|\le \sqrt{\frac{2\log{(2/\rho)}}{m(\bar p_i +...+\bar p_S)}}.\]
 Therefore, with probability $1-S\rho$,
\begin{eqnarray}
\label{eq:part1}
& & (\tilde{p}-\bar p)^Th - \sum_i (\tilde y_i - \bar{y}_i)({h}_i-\bar{H}_{i+1})(\bar{p}_i+\cdots + \bar p_S) \nonumber\\
& = & \sum_i (\tilde y_i - \bar{y}_i)(\tilde H_{i+1}-\bar{H}_{i+1})(\bar{p}_i+\cdots + \bar p_S)\nonumber\\
& \ge & - \sum_i \sqrt{\frac{2\log(2/\rho)}{m(\bar p_i +...+\bar p_S)}} D\sqrt{\frac{2\log(2/\rho)}{m(\bar p_i +...+\bar p_S)}} (\bar{p}_i+\cdots + \bar p_S)\nonumber\\
& \ge & -\frac{2SD\log(2/\rho)}{m}.
\end{eqnarray}
Then, applying Lemma \ref{lem:Anti1} (given $m\bar p_i\ge 6$) for $z_i = ({h}_i-\bar{H}_{i+1})(\bar{p}_i+\cdots + \bar p_S), i=1,\ldots, S$, with probability $\Omega(1/S)$,
\[(\tilde{p}-\bar p)^Th \ge  \frac{1}{4} \sqrt{\sum_i z_i^2 \bar\sigma_i^2} -\frac{2SD\log(2/\rho)}{m}. \]
Now, we observe 
\[\sum_i z_i^2 \bar\sigma_i^2 = ({h}_i-\bar{H}_{i+1})^2 (\bar{p}_i+\cdots + \bar p_S)^2 \bar\sigma_i^2 =  \frac{\bar c_i^2 \bar p_i (\bar p_i +\ldots, \bar p_S)}{m(\bar p_i + \ldots + \bar p_S)+1},\] 
to obtain that with probability at least $\Omega(1/S) - S\rho$,
\[(\tilde{p}-\bar p)^Th \ge  \frac{1}{8}   \sqrt{\sum_i \frac{\bar \gamma_i \bar c_i^2}{m}} -\frac{2SD\log(2/\rho)}{m}\]
where 
\[\bar \gamma_i = \frac{\bar p_i (\bar p_{i+1} +\ldots + \bar p_S)}{(\bar p_i + \ldots + \bar p_S)}.\]

\end{myproof}

\subsection{Concentration of empirical probability vectors: Proof of Proposition \ref{prop:ConcentrationFixedh}}
\label{app:deviation_emp}

\concEmp*
\begin{myproof}
For every $t,i$, define 
\[Z_{t,i} = \left( c_i \indi(s_{t}=i) - c_i \frac{p_i}{p_i + \cdots + p_S} \cdot \indi(s_t \in \{i,\ldots, S\})\right) \indi(s_{t-1}=s, a_{t-1}=a),\]
\[Z_t=\sum_{i} Z_{t,i}.\]
Then, 
\[ \frac{\sum_{t=1}^\tau Z_t}{n} = \sum_i c_i \hat p_i  - \sum_{i} \frac{c_i p_i}{p_i + \cdots + p_S} \cdot (\hat p_i +\ldots + \hat p_S) = \sum_{i=1}^{S-1} (\hat{y}_i - y_i) (\hat p_i + \ldots + \hat p_S) c_i = (\hat{p}-p)^Th\]
where we used Fact \ref{fact:DirRepresent} for the last equality.
Now, $E[Z_t | s_{t-1}, a_{t-1}] = \sum_{i} E[Z_{t,i}| s_{t-1}, a_{t-1}]=0$. Also, we observe that for any $t$, $Z_{t,i}$ and $Z_{t,j}$ for any $i\ne j$ are independent given the current state and action: (assume $j>i$ w.l.o.g.)
\begin{eqnarray*}
\Ex[Z_{t,i} Z_{t,j} |s_{t-1}, a_{t-1}] & = & c_ic_j \Ex[ \indi(s_{t}=i) \indi(s_{t}=j)-  \indi(s_{t}=j) \frac{p_i}{p_i + \cdots + p_S} \cdot \indi(s_t \in \{i,\ldots, S\}) \\
& & - \indi(s_{t}=i) \frac{p_j}{p_j + \cdots + p_S} \cdot \indi(s_t \in \{j,\ldots, S\}) \\
& & + \frac{p_j p_i}{(p_j + \cdots + p_S)(p_i + \cdots + p_S)} \cdot \indi(s_t \in \{j,\ldots, S\}) ]\\
& = & c_ic_j \Ex[-  \indi(s_{t}=j) \frac{p_i}{p_i + \cdots + p_S} \\
&  & + \frac{p_j p_i}{(p_j + \cdots + p_S)(p_i + \cdots + p_S)} \cdot \indi(s_t \in \{j,\ldots, S\}) ]\\
& = & c_ic_j \Ex[- \frac{p_j p_i}{p_i + \cdots + p_S} + \frac{p_j p_i}{(p_i + \cdots + p_S)} ]\\
& =& 0.
\end{eqnarray*}
Therefore,
\[\sum_{t=1}^\tau E[Z_t^2 | s_{t-1}, a_{t-1}] = \sum_{t=1}^\tau \sum_{i} c_i^2 \Ex[Z_{t,i}^2 | s_{t-1}, a_{t-1}] = \sum_{i}c_i^2 n  \gamma_i,\]
where the last equality is obtained using the following derivation:
\begin{eqnarray*}
\Ex[\sum_{t=1}^\tau Z_{t,i}^2 | s_{t-1}=s, a_{t-1}=a] & = & \sum_{t=1}^\tau \indi(s_{t-1}=s, a_{t-1}=a) \left(p_i -  \frac{p_i^2}{(p_i + \cdots + p_S)^2} (p_i+\cdots+p_S) \right)\\
& = & \sum_{t=1}^\tau \indi(s_{t-1}=s, a_{t-1}=a)  \frac{p_i(p_{i+1} + \cdots + p_S)}{p_i+\cdots+p_S} \\
& = & n\frac{p_i(p_{i+1} + \cdots + p_S)}{p_i+\cdots+p_S} = n\gamma_i.
\end{eqnarray*}
Then, applying Bernstein's inequality (refer to Corollary \ref{cor:bim}) to bound $|\sum_{t=1}^\tau Z_t|$, we get the desired bound on $(p-\hat p)^Th = \frac{1}{n} \sum_{t=1}^\tau Z_t $.
\end{myproof}



\subsection{A modified anti-concentration bound: Proof of Proposition \ref{prop:antiMain2}}
We use the notation described in Section \ref{subsec:notation}. Given an epoch $k$, state $s$, action $a$, and sample $j$, we denote $n=N_{s,a}^{\tau_k}$, $n_i=N_{s,a}^{\tau_k}(i)$, $m=\frac{n+\omega S}{\boost}$, where $\omega=\omegaValue$ and $\boost=\boostValue$, as defined in the algorithm. Then, we denote $p_i=P_{s,a}(i)$, $\hat{p}_i := \frac{n_i}{n}$, $\bar{p}_i = \frac{n_i+\omega}{n+\omega S}$, and $\tilde{p}_i=Q^{j,k}_{s,a}(i)$, for $i\in \st$.

Also, as defined earlier in Proposition \ref{prop:antiMain} and Proposition \ref{prop:ConcentrationFixedh}, we denote 
$$\bar \gamma_i := \frac{\bar p_i (\bar p_{i+1} +\ldots + \bar p_S)}{(\bar p_i + \ldots + \bar p_S)}, \bar c_i = (h_i - \bar H_{i+1}), \bar H_{i+1} = \frac{1}{\sum_{j=i+1}^S \bar{p}_j}\sum_{j=i+1}^S h_j \bar{p}_j,$$
and 
$$\gamma_i=\frac{p_i(p_{i+1} + \cdots + p_S)}{(p_i+ \cdots +p_S)}, c_i= h_i-H_{i+1}, H_{i+1}=\frac{1}{\sum_{j=i+1}^S {p}_j}\sum_{j=i+1}^S h_j {p}_j$$

We prove the following result for $s,a$ such that $n>\smallN$. Recall that for such $s,a$, the algorithm uses Dirichlet posterior sampling to generate sample vectors $Q^{j,k}_{s,a}$, so that in this case $\tilde{p}$ is a random vector distributed as $\Dir(m\bar p_1, \ldots, m \bar p_S)$. 

\begin{proposition}
\label{prop:antiMain2}
Assume that $h \in [0,D]^S$, and $\omega \ge 613 \log(2/\rho), n > 12 \omega S^2, \boost = \frac{\omega}{6}$, and an ordering of $i$ such that $\bar{p}_1 \leq \cdots \leq \bar p_S$. Then, with probability $\Omega(1/S) - 7S\rho$,
\[(\tilde{p}-\bar p)^Th \ge 0.184 \sqrt{\boost \sum_i \frac{\gamma_i c_i^2}{n}} - O(\frac{DS\omega \log(n/\rho)}{n}).\] 
\end{proposition}
\begin{myproof}
The proof is obtained by a modification to the proof of Proposition \ref{prop:antiMain}, which proves a similar bound but in terms of $\bar \gamma_i$'s and $\bar c_i$'s and $m$. 

In the proof of that proposition, we obtain (refer to Equation \eqref{eq:part1}),  with probability $1-S\rho$ (assuming $m\bar{p}_i \geq 6$),
\begin{eqnarray*}
(\tilde{p}-\bar p)^Th & \ge &  \sum_i (\tilde y_i - \bar{y}_i)({h}_i-\bar{H}_{i+1})(\bar{p}_i+\cdots + \bar p_S) -\frac{2DS\log(2/\rho)}{m} \\
&\geq& \sum_i (\tilde y_i - \bar{y}_i)({h}_i-\bar{H}_{i+1})(\bar{p}_i+\cdots + \bar p_S)-O(\frac{DS\omega \log(n/\rho)}{n})
\end{eqnarray*}
where $\tilde{y}_i:=\frac{\tilde p_i}{\tilde p_{i}+\cdots + \tilde p_S},  \bar{y}_i:=\frac{\bar p_i}{\bar p_{i}+\cdots + \bar p_S}$, $\tilde{H}_{i+1}= \frac{1}{\sum_{j=i+1}^S \tilde{p}_j}\sum_{j=i+1}^S h_j \tilde{p}_j$, $\bar{H}_{i+1}= \frac{1}{\sum_{j=i+1}^S \bar{p}_j}\sum_{j=i+1}^S h_j \bar{p}_j$. 
Now, breaking up the term in the summation and using Lemma \ref{lemm:Hdiff} to bound $|H_{i+1} - \bar H_{i+1} |(\bar{p}_i+\cdots + \bar p_S)$ (since we have by assumption that $\omega \ge 613\log(2/\rho)$ and $n > 12 \omega S^2$) and Lemma \ref{gvb} and Corollary \ref{cor:soconc1} to bound $|\tilde y_i - \bar{y}_i|$ (see proof of Proposition \ref{prop:antiMain}), we get that for every $i$, with probability $1-4S\rho$,

\begin{eqnarray*}
& & (\tilde{p}-\bar p)^Th - \sum_i (\tilde y_i - \bar{y}_i)({h}_i-{H}_{i+1})(\bar{p}_i+\cdots + \bar p_S)  + O(\frac{DS\omega\log(n/\rho)}{m}) \\
&\ge & \sum_i (\tilde y_i - \bar{y}_i)(\bar {H}_{i+1}-{H}_{i+1})(\bar{p}_i+\cdots + \bar p_S)\\
& \ge & - \sum_i \sqrt{\frac{2\log(2/\rho)}{m(\bar p_i +\cdots+\bar p_S)}} \left(3 D\sqrt{\log(n/\rho) \frac{(\bar p_i + \cdots +\bar  p_S)}{n} } + 4\frac{(\omega S+\log(n/\rho)) D}{n}\right)\\
(*) & \ge & -\frac{6DS\sqrt{\log(2/\rho)\log(n/\rho)}}{\sqrt{mn}} - \frac{4(\omega S+\log(n/\rho)) D\sqrt{2\log(2/\rho)}}{n\sqrt{m}}\sum_i\frac{1}{ \sqrt{(\bar{p}_i+\cdots+\bar{p}_S)}}.
\end{eqnarray*}

Then, applying Lemma \ref{lem:Anti1} (given $m\bar p_i\ge 6$) for $z_i = ({h}_i-{H}_{i+1})(\bar{p}_i+\cdots + \bar p_S), i=1,\ldots, S$, with probability $\Omega(1/S)$,
\[\sum_i (\tilde y_i - \bar{y}_i)z_i \ge \frac{1}{4} \sqrt{\sum_i \bar\sigma_i^2z_i^2}.\]
We substitute this in the above, with the observation
\[\sum_i z_i^2 \bar\sigma_i^2 = \sum_i ({h}_i-H_{i+1})^2 (\bar{p}_i+\cdots + \bar p_S)^2 \bar\sigma_i^2 =  \sum_i\frac{c_i^2 \bar p_i (\bar p_i +\ldots, \bar p_S)}{m(\bar p_i + \ldots + \bar p_S)+1} \geq \sum_i \frac{6}{7}\frac{\bar \gamma_i c_i^2}{m}.\] 
So far we have that with probability $\Omega(1/S) -4S\rho$,
\begin{equation}
\label{eq:statement1}
(\tilde{p}-\bar p)^Th \ge  \frac{\sqrt{6}}{4\sqrt{7}}   \sqrt{\sum_i \frac{\bar \gamma_i c_i^2}{m}} -O(\frac{DS\omega \log(n/\rho)}{n}).
\end{equation}

Finally, we  use Lemma \ref{lem:gammaReplace} with $k=14$ (this requires $\omega \geq 613\log(2/\rho)$) to lower bound $\bar \gamma_i$ by $\frac{1}{1.51}\gamma_i - O(\frac{\omega S}{n})$ to get with probability $\Omega(1/S) - 7S\rho$,
\begin{eqnarray*}
(\tilde{p}-\bar p)^Th &\ge&  0.188 \sqrt{\sum_i \frac{\gamma_i c_i^2}{m}} -O(\frac{DS\omega \log(n/\rho)}{n}).
\end{eqnarray*}


Recall that $m = \frac{n+\omega S}{\boost}$, so that for $n > S\omega$, $n \geq \frac{m\boost}{2} = \frac{m\omega}{12} \geq m\log(2/\rho)$, and the first term of $(*)$ is at least:
\[-\frac{6DS\sqrt{\log(2/\rho)\log(n/\rho)}}{\sqrt{m^2\log(2/\rho)}} = -\frac{6DS\sqrt{\log(n/\rho)}}{m} = -O(\frac{DS\omega \log(n/\rho)}{n}).\]

Then using Lemma \ref{lem:sumpi} and $m = (n +S\omega)/\boost > 6n/\omega > 72S^2$, the second term in $(*)$ is at least:
\[ - \frac{8S(\omega S+\log(n/\rho)) D\sqrt{2\log(2/\rho)}}{n\sqrt{72S^2}} = -O(\frac{DS\omega \log(n/\rho)}{n}).\]

\end{myproof}
\begin{lemma}
\label{lem:sumpi}
Let $x \in \mathbb{R}^n$ such that $0 \leq x_1 \leq \cdots \leq x_n \leq 1$ and $\sum_i x_i = 1$. Then \[\sum_{i=1}^n \frac{1}{\sqrt{x_i + \cdots x_n}} \leq 2n.\]
\end{lemma}
\begin{myproof} 
Define $f(y) := \frac{1}{\sqrt{x_y +\cdots + x_n}}$ for all $y=1,\cdots,n$. We prove that $x^* := (\frac{1}{n}, \frac{1}{n}, \cdots, \frac{1}{n})$ achieves the maximum value. Consider any solution $x'$. Assume there exists some index pair $i,j$ with $i<j$ and some $\epsilon>0$ such that $x'_i \neq x'_j$ and increasing $x'_i$ by $\epsilon$ and decreasing $x'_j$ by $\epsilon$ preserves the ordering of the indices. This strictly increases the objective, because $f(k)$ strictly increases for all $i<k\leq j$ and remains unchanged otherwise, and hence $x'$ is not an optimal solution. The only case where no such index pair $(i,j)$ exists is when every $x_i$ is equal- this is precisely the solution $x^*$. Since $\sum_i f(i)$ is a continuous functions over a compact set, it has a maximum, which therefore must be attained at $x^*$. 

This means \[\sum_{i=1}^n \frac{1}{\sqrt{x_i + \cdots x_n}} \leq \sum_{i=1}^n \frac{1}{\sqrt{x_i^* + \cdots + x_n^*}} = \sum_{i=1}^n \sqrt{\frac{n}{i}} \leq \sqrt{n}\int_{i=0}^n \frac{1}{\sqrt{i}} di = 2n.\]
\end{myproof}

\begin{lemma}
\label{lem:gammaReplace}
Let $A =3\log(\frac{2}{\rho}) $ and $\omega \ge \frac{25}{24}k^2A$. Also let $n > 12 \omega S^2$. Then for any group ${\cal G}$ of indices, with probability $1-\rho$,
\[(1-\frac{1}{k}) \sum_{i \in {\cal G}} \bar{p_i} - \frac{2\omega S}{n} \le \sum_{i \in {\cal G}} p_i \le  (1+\frac{1}{k})\sum_{i \in {\cal G}} \bar{p_i} + \frac{2\omega S}{n}.\]
If in the definition of $\bar \gamma_i$, we use an ordering of $i$ such that $\bar p_S \ge \frac{1}{S}$ (e.g., if $\max \bar p_i$ is the last in the ordering), then for all $i$, with probability $1-3\rho$,
\[\gamma_i \le \frac{(1+\frac{1}{k})^2 }{1-\frac{1}{k}-\frac{1}{6}} \bar \gamma_i + \frac{2(1+\frac{1}{k}+\frac{1}{6})}{1-\frac{1}{k}-\frac{1}{6}} \frac{\omega S}{n}.\]

\end{lemma}
\begin{myproof}
By multiplicative Chernoff bounds (Fact \ref{mcher}), with probability $1-\rho$, 
\[|\sum_i p_i - \sum_i \hat p_i| \le \sqrt{\frac{A \sum_i\hat p_i}{n}} + \frac{A}{n}\]
where $A=3\log(\frac{2}{\rho})$
so that using $|\sum_i\bar p_i - \sum_i\hat p_i| \le \frac{\omega S}{n}$,
\[|\sum_ip_i - \sum_i\bar p_i| \le \sqrt{\frac{\sum_i\bar p_i A}{n}} + \frac{\sqrt{A \omega S}}{n} + \frac{A}{n} + \frac{\omega S}{n} \leq  \sqrt{\frac{\sum_i\bar p_i A}{n}} + \frac{2\omega S}{n}.\]

Now, for $n > 12\omega S^2$, $n \bar{p}_i = n\frac{n\hat{p}_i +\omega}{n+\omega S} \geq \frac{n\omega}{n+\omega S} \geq \frac{24\omega}{25} \ge k^2A$.
\[ |\sum_i p_i - \sum_i \bar{p_i}| \le  \sum_i\bar p_i \sqrt{\frac{A}{n \sum_i \bar p_i}} + \frac{ 2\omega S}{n} \le \sum_i \bar{p_i}  \sqrt{\frac{A}{k^2A}} + \frac{2\omega S}{n}  \le \frac{1}{k} \sum_i\bar{p_i}  + \frac{2\omega S}{n} \]
so that 
\[\sum_i p_i \le  (1+\frac{1}{k})\sum_i\bar{p_i} + \frac{2\omega S}{n}, \ \ \sum_i p_i \ge  (1-\frac{1}{k}) \sum_i\bar{p_i} - \frac{2\omega S}{n}.\]
For the second statement of the lemma, using what we just proved, we have that with probability $1-3\rho$,
\[\gamma_i = \frac{p_i(p_{i+1} + \cdots + p_S)}{p_i+\cdots+p_S} \le  \frac{(1+\frac{1}{k})^2 \bar p_i(\bar p_{i+1} + \cdots + \bar p_S) + \frac{2(1+\frac{1}{k})\omega S (\bar p_i+\cdots+ \bar p_S)}{n}+\frac{4\omega^2 S^2}{n^2}}{(1-\frac{1}{k}) (\bar p_i+\cdots+\bar p_S) - \frac{2\omega S}{n}}.\]
Now, if indices $i$ are ordered such that $\bar p_S \ge \frac{1}{S}$, then
$\bar p_{i} + \cdots + \bar p_S \ge \frac{1}{S}$ for all $i$. Also, if $n > 12 \omega S^2$,  we have the following bound on the denominator in above: 
$(1-\frac{1}{k}) (\bar p_i+\cdots+ \bar p_S) - \frac{2\omega S}{n} \ge (1-\frac{1}{k}-\frac{1}{6}) (\bar p_i+\cdots+ \bar p_S)$, so that from above 
\[\gamma_i \le \frac{(1+\frac{1}{k})^2 }{1-\frac{1}{k}-\frac{1}{6}} \bar \gamma_i + \frac{2(1+\frac{1}{k}+\frac{1}{6})}{1-\frac{1}{k}-\frac{1}{6}} \frac{\omega S}{n}.\]
\end{myproof}

\begin{lemma}
\label{lemm:Hdiff}
For any {\bf fixed} $h \in \mathbb{R}^S$,  and $i$, let $\hat H_{i}=\frac{1}{\sum_{j=i}^S {\hat p}_j}\sum_{j=i}^S h_j \hat{p}_j$, $H_{i}=\frac{1}{\sum_{j=i}^S {p}_j}\sum_{j=i}^S h_j {p}_j$, $\bar H_{i}=\frac{1}{\sum_{j=i}^S {\bar p}_j}\sum_{j=i}^S h_j \bar{p}_j$. Then if $n\geq 96$, with probability $1-\rho$, 
\[|(\bar H_i - H_i)(\bar p_i +\ldots + \bar p_S)| \le  2 D\sqrt{\log(n/\rho) \frac{(p_i + \cdots + p_S)}{n} } + 3\frac{(\omega S+\log(n/\rho)) D}{n}.\]

Moreover, if we also assume that $\omega \ge 30\log(2/\rho)$ and $n > 12 \omega S^2$, then with probability $1-2\rho$, 
\[|(\bar H_i - H_i)(\bar p_i +\ldots + \bar p_S)| \le  3 D\sqrt{\log(n/\rho) \frac{(\bar p_i + \cdots +\bar  p_S)}{n} } + 4\frac{(\omega S+\log(n/\rho)) D}{n}.\]
\end{lemma}
\begin{myproof}
For every $t,k\ge i$, define 
\[Z_{t,k} = \left( h_k \indi(s_{t}=k) - h_k \frac{p_k}{p_i + \cdots + p_S} \cdot \indi(s_t \in \{i,\ldots, S\})\right) \indi(s_{t-1}=s, a_{t-1}=a),\]
\[Z_t=\sum_{k\ge i} Z_{t,k}.\]
Then, 
\[ \frac{\sum_{t=1}^\tau Z_t}{n} = \sum_{k\ge i} h_k \hat p_k  - \sum_{k\ge i} h_k \frac{p_k}{p_i + \cdots + p_S} \cdot (\hat p_i +\ldots + \hat p_S) = (\hat H_i - H_i)(\hat p_i +\ldots + \hat p_S)\]
where we used Fact \ref{fact:DirRepresent} for the last equality.
Now, $E[Z_t | s_{t-1}, a_{t-1}] = \sum_{k\ge i} E[Z_{t,k}| s_{t-1}, a_{t-1}]=0$. Also, we observe that for any $t$, $Z_{t,k}$ and $Z_{t,j}$ for any $k\ne j$ are negatively correlated given the current state and action: 
\begin{eqnarray*}
\Ex[Z_{t,k} Z_{t,j} |s_{t-1}, a_{t-1}] & = & h_kh_j \Ex[ \indi(s_{t}=k) \indi(s_{t}=j)-  \indi(s_{t}=j) \frac{p_k}{p_i + \cdots + p_S} \cdot \indi(s_t \in \{i,\ldots, S\}) \\
& & - \indi(s_{t}=k) \frac{p_j}{p_i + \cdots + p_S} \cdot \indi(s_t \in \{i,\ldots, S\}) \\
& & + \frac{p_j p_k}{(p_i + \cdots + p_S)^2} \cdot \indi(s_t \in \{i,\ldots, S\}) ]\\
& = & h_kh_j \Ex[-  \frac{2p_jp_k}{p_i + \cdots + p_S} + \frac{p_k p_j}{(p_i + \cdots + p_S)^2} \cdot \indi(s_t \in \{i,\ldots, S\}) ]\\
& = & h_kh_j \Ex[- \frac{p_j p_i}{p_i + \cdots + p_S}] \\
& \le & 0.
\end{eqnarray*}
And,
\begin{eqnarray*}
\Ex[\sum_{t=1}^\tau Z_{t,k}^2 | s_{t-1}=s, a_{t-1}=a] & = & h_k^2 \sum_{\tau=1}^t \indi(s_{t-1}=s, a_{t-1}=a) \left(p_k -  \frac{p_k^2}{(p_i + \cdots + p_S)^2} (p_i+\cdots+p_S) \right)\\
& = & h_k^2 \sum_{t=1}^\tau \indi(s_{t-1}=s, a_{t-1}=a)  \frac{p_k(\sum_{j\ge i, j\ne k}p_j)}{p_i+\cdots+p_S} \\
& = & n h_k^2\frac{p_k(\sum_{j\ge i, j\ne k} p_{j})}{p_i+\cdots+p_S}\\
& \le & n D^2 p_k.
\end{eqnarray*}
Therefore,
\[\sum_{t=1}^\tau E[Z_t^2 | s_{t-1}, a_{t-1}] \le \sum_{t=1}^\tau \sum_{k\ge i} \Ex[Z_{t,k}^2 | s_{t-1}, a_{t-1}] \le nD^2 (p_i + \cdots + p_S). \]
Then, applying Bernstein's inequality (refer to Corollary \ref{cor:bim}) to bound $|\sum_{t=1}^\tau Z_t|$, we get the following bound on $\frac{1}{n} \sum_{t=1}^\tau Z_t = (\hat H_i - H_i)(\hat p_i +\ldots + \hat p_S)$ with probability $1-\rho$:
\[|(\hat H_i - H_i)(\hat p_i +\ldots + \hat p_S)| = |\frac{1}{n} \sum_{t=1}^\tau Z_t| \le 2 D\sqrt{\log(n/\rho) \frac{(p_i + \cdots + p_S)}{n} } + 3D\frac{\log(n/\rho)}{n}.\]
Also,
\[|\hat H_i - \bar H_i|= |\sum_k \frac{\hat p_k}{\hat p_i+\cdots+\hat p_S} h_k  - \frac{\bar p_k}{\bar p_i+\cdots+\bar p_S}h_k |  \le \frac{\omega S D}{n(\hat p_i+\cdots+\hat p_S)},\] 
Combining,
\[|(\bar H_i - H_i)(\hat p_i +\ldots + \hat p_S)| \le 2 D\sqrt{\log(n/\rho) \frac{(p_i + \cdots + p_S)}{n} } + 3D\frac{\log(n/\rho)}{n} + \frac{\omega S D}{n}.\]
Replacing $\hat p_i$ by $\bar p_i$, 
\[|(\bar H_i - H_i)(\bar p_i +\ldots + \bar p_S)| \le  2 D\sqrt{\log(n/\rho) \frac{(p_i + \cdots + p_S)}{n} } + 3\frac{(\omega S+\log(n/\rho)) D}{n}\]
with probability $1-\rho$.

Now, if we also have that $\omega \ge 30\log(2/\rho)$ and $n > 12 \omega S^2$, using lemma \ref{lem:gammaReplace} with $k=3$ to replace $p_i$ by $\bar p_i$, with probability $1-2\rho$,
\[|(\bar H_i - H_i)(\bar p_i +\ldots + \bar p_S)| \le  3 D\sqrt{\log(n/\rho) \frac{(\bar p_i + \cdots +\bar  p_S)}{n} } + 4\frac{(\omega S+\log(n/\rho)) D}{n}.\]

\end{myproof}

\section{Useful deviation inequalities}
\label{app:Inequalities}

\begin{fact}[Bernstein's Inequality, from \cite{seldin2012pac} Lem 11/Cor 12]
\label{bim}
 Let $Z_1, Z_2, ..., Z_n$ be a bounded martingale difference sequence so that $|Z_i| \leq K$ and $\Ex[Z_i | {\cal F}_{i-1}] = 0$. Define $M_n = \sum_{i=1}^n Z_i$ and $V_n = \sum_{i=1}^n \Ex[(Z_i)^2 | {\cal F}_{i-1}]$. For any $c>1$ and $\delta \in (0,1)$, with probability greater than $1-\delta$, if \[\sqrt{\frac{ \ln{\frac{2\nu}{\delta}} }{(e-2)V_n}} \leq \frac{1}{K}\] then \[|M_n| \leq (1+c)\sqrt{(e-2)V_n\ln{\frac{2\nu}{\delta}}},\] otherwise, \[|M_n| \leq 2K\ln{\frac{2\nu}{\delta}},\]
where \[\nu = \lceil \frac{\ln{(\sqrt{\frac{(e-2)n}{\ln{\frac{2}{\delta}}}})}}{\ln{c}} \rceil + 1.\]

\end{fact}

\begin{corollary}[to Bernstein's Inequality above]
\label{cor:bim}
 Let $Z_i$ for $i = 1,\cdots, n$, $M_n$, and $V_n$ as above. For $n\geq 96$ and $\delta \in (0,1)$, with probability greater than $1-\delta$,  
\[|M_n| \leq 2\sqrt{V_n\ln{\frac{n}{\delta}}} + 3K\ln{\frac{n}{\delta}}.\]
\end{corollary}

\begin{myproof} 
Applying Bernstein's Inequality above with $c=1+\frac{4}{n}$, with probability greater than $1-\delta$, 
\begin{eqnarray*}
|M_n| &\leq& (1+c)\sqrt{(e-2)V_n\ln{\frac{2\nu}{\delta}}} + 2K\ln{\frac{2\nu}{\delta}}  \\
&\le & (1+c)\sqrt{(e-2)V_n\ln{\frac{n^{\frac{4}{3}}}{\delta}}} + 2K\ln{\frac{n^{\frac{4}{3}}}{\delta}} \\
&\le & (1+c)\sqrt{(e-2)\frac{4}{3}V_n\ln{\frac{n}{\delta}}} + 3K\ln{\frac{n}{\delta}} \\
&\le & 2\sqrt{V_n\ln{\frac{n}{\delta}}} + 3K\ln{\frac{n}{\delta}}
\end{eqnarray*}
where \[\nu = \lceil \frac{\ln{(\sqrt{\frac{(e-2)n}{\ln{\frac{2}{\delta}}}})}}{\ln{c}} \rceil + 1  = \lceil \frac{n}{2}\ln{(\sqrt{\frac{(e-2)n}{\ln{\frac{2}{\delta}}}})} \rceil +1 \leq  \frac{n}{2}\ln{(\sqrt{\frac{(e-2)n}{\ln{2}}})} +2 \leq \frac{1}{2}n^{\frac{4}{3}}.\]
\end{myproof}

\begin{fact}[Multiplicative Chernoff Bound, \cite{kleinberg2008multi} Lemma 4.9]
\label{mcher} 
Consider $n$ i.i.d. random variables $X_1, \cdots, X_n$ on $[0,1]$. Let $\mu$ be their mean and let $X$ be their average. Then for any $\alpha>0$ the following holds:
\[P(|X-\mu| < r(\alpha,X)<3r(\alpha,\mu)) > 1- e^{\Omega(\alpha)},\]
where $r(\alpha,x) = \sqrt{\frac{\alpha x}{n}} + \frac{\alpha}{n}.$

More explicitly, we have that with probability $1-\rho$, 
\[|X-\mu| <  \sqrt{\frac{3\log(2/\rho) X}{n}} + \frac{3\log(2/\rho)}{n}.\]
\end{fact}

\comment{
\begin{corollary}
Consider $n$ i.i.d. random variables $X_1, \cdots, X_n$ on $[0,1]$. Let $\mu$ be their mean and let $X$ be their average. Then
\[P(|X-\mu| < \sqrt{\frac{3X\log(2/\rho)}{n}} + \frac{3\log(2/\rho)}{n}) \geq  1- \rho.\]
\end{corollary}
\begin{myproof}
Let $L= \log(2/\rho)$ and $\delta = \frac{3L}{n\mu} >0$. First assume $
\delta >1$. Then by Chernoff bounds, \[P(X - \mu \geq \frac{3L}{n})\leq \rho.\] Since $\delta >1$, we have that $\mu < \frac{3L}{n}$ so $P(X-\mu < -\frac{3L}{n}) = 0$. 
Then \[P(|X-\mu| \leq \frac{3L}{n}) \geq 1-\rho.\]
On the other hand, if $0 < \delta < 1$ then $0< \sqrt{\delta} <1$, so again by Chernoff bounds we have that with probability $1-\rho$, \[P(|X-\mu| \leq \sqrt{\frac{3\mu L}{n}}).\]

Note that the condition $|X-\mu| \leq \sqrt{\frac{3\mu L}{n}}$ implies that $|X-\mu|  \leq \sqrt{\frac{3XL}{n}} + \frac{3L}{n}$:
\begin{eqnarray*}
|X-\mu| &\leq& \sqrt{\frac{3\mu L}{n}} \leq \sqrt{\frac{3(X+|\mu-X|) L}{n}} \\
|X-\mu|^2 &\leq& \frac{3(X+|\mu-X|) L}{n} = \frac{3XL}{n} +\frac{3|\mu-X|L}{n} \\
|X-\mu|^2 -  \frac{3|X-\mu|L}{n} + (\frac{3L}{2n})^2 &\leq& \frac{3XL}{n} + (\frac{3L}{2n})^2 \\
(|X-\mu| - \frac{3L}{2n})^2 &\leq& \frac{3XL}{n} + (\frac{3L}{2n})^2 \\
|X-\mu|  &\leq& \sqrt{\frac{3XL}{n} + (\frac{3L}{2n})^2} + \frac{3L}{2n} \leq \sqrt{\frac{3XL}{n}} + \frac{3L}{n}
\end{eqnarray*}
Therefore, for all $\delta >0$, we get that with probability $1-\rho$, \[|X-\mu|  \leq \sqrt{\frac{3XL}{n}} + \frac{3L}{n}.\]
\end{myproof}
}


\begin{fact}[Cantelli's Inequality]
\label{cant} Let $X$ be a real-valued random variable with expectation $\mu$ and variance $\sigma^2$. Then $P(X - \mu \geq \lambda)  \leq \frac{\sigma^2}{\sigma^2 + \lambda^2}$ for $\lambda >0$ and $P(X - \mu \geq \lambda) \geq 1-\frac{\sigma^2}{\sigma^2 + \lambda^2}$ for $\lambda < 0$.
\end{fact}

\begin{fact}[Berry-Esseen Theorem] 
\label{be}
Let $X_1, X_2, ...,X_n$ be independent random variables with $\Ex[X_i] = 0$, $\Ex[X_i^2]= \sigma_i^2 > 0$, and $\Ex[|X_i|^3] = \rho_i < \infty$. Let \[S_n = \frac{X_1+X_2+...+X_n}{\sqrt{\sigma_1^2+...+\sigma_n^2}}\] and denote $F_n$ the cumulative distribution function of $S_n$ and $\Phi$ the cumulative distribution function of the standard normal distribution. Then for all $n$, there exists an absolute constant $C_1$ such that \[sup_{x \in R} |F_n(x) - \Phi(x)| \leq C_1\psi_1\] where $\psi_1 = (\sum\limits_{i=1}^n \sigma_i^2)^{-1/2} \max_{1\leq i \leq n}\frac{\rho_i}{\sigma_i^2}$. The best upper bound on $C_1$ known is $C_1 \leq 0.56$ (see \cite{shevtsova2010improvement}).
\end{fact}
\begin{fact}[\cite{abramowitz1964handbook} 26.5.21] 
\label{betanormal}
Consider the regularized incomplete Beta function $I_z(a,b)$ (cdf) for the Beta random variable with parameters $(a,b)$. 
For any $z$ such that $(a+b-1)(1-z) \geq 0.8$,  $I_z(a,b) = \Phi(y) + \epsilon$, with $|\epsilon| < 0.005$ if $a+b >6$. Here $\Phi$ is the standard normal CDF with \[y = \frac{3[w_1(1-\frac{1}{9b}) - w_2(1-\frac{1}{9a})]}{[\frac{w_1^2}{b} + \frac{w_2^2}{a}]^{1/2}},\] where $w_1 = (bz)^{1/3}$ and $w_2 = [a(1-z)]^{1/3}$. 
\end{fact}
The following lemma uses the above fact to lower bound the probability of a Beta random variable to exceed its mean by a quantity close to its standard deviation.
\begin{lemma}[Anti-concentration for Beta Random Variables]
\label{lem:BetaAntiConcentration}
Let $F_{a,b}$ denote the cdf of a Beta random variable with parameter $(a,b)$, with $a\ge 6, b\ge 6$. Let $z=\frac{a}{a+b} + C\sqrt{\frac{ab}{(a+b)^2(a+b+1)}} + \frac{C}{a+b}$, with $C\le 0.5$. Then,
\[1-F_{(a,b)}(z) \ge 1-\Phi(1) - 0.005 \ge 0.15.\]
\end{lemma}
\begin{myproof}
Let $x=C\sqrt{\frac{ab}{(a+b+1)}}+C$. Then, $z=\frac{a+x}{a+b}$,$w_1 = (b(a+x)/(a+b))^{1/3}$ and $w_2 = [a(b-x)/(a+b))]^{1/3}$. Also, $z\le 2C\sqrt{\frac{ab}{a+b}}$. Also, $(a+b-1)(1-z) \geq  (a+b-1)(1-\frac{a}{a+b} - C\sqrt{\frac{ab}{(a+b)^2(a+b+1)}} - \frac{C}{a+b}) =  (a+b-1)(\frac{b}{a+b} - \frac{C}{a+b}\sqrt{\frac{ab}{a+b+1}} - \frac{C}{a+b})\geq \frac{a+b-1}{a+b}(b-C\sqrt{\frac{ab}{a+b+1}}-\frac{C}{a+b}) \geq \frac{11}{12}(b-C\sqrt{b}-\frac{C}{12}) \geq 0.8$. Hence we can apply Fact \ref{betanormal} relating Beta with Normal.
We bound the numerator and denominator in the expression of $y$, to show that the relation $I_z(a,b) \le \Phi(y) + \epsilon$ holds for some $y\le 1$. 
\begin{eqnarray*}
numerator(y) & = & 3[w_1(1-\frac{1}{9b}) - w_2(1-\frac{1}{9a})]\\
& = & 3(\frac{ab}{a+b})^{\frac{1}{3}}[(1+\frac{x}{a})^{\frac{1}{3}}(1-\frac{1}{9b}) - (1-\frac{x}{b})^{\frac{1}{3}}(1-\frac{1}{9a})]\\
& \leq & 3(\frac{ab}{a+b})^{\frac{1}{3}}[(1+\frac{x}{3a})(1-\frac{1}{9b}) - (1-\frac{x}{3b} - \frac{2x^2}{9b^2})(1-\frac{1}{9a})]\\
& = & 3(\frac{ab}{a+b})^{\frac{1}{3}}[(\frac{b-a}{9ab}) + (\frac{x(a+b)}{3ab}) - (\frac{2x}{27ab})] + 3(\frac{ab}{a+b})^{\frac{1}{3}}[\frac{2x^2}{9b^2}(1-\frac{1}{9a})]\\
& \leq & 3(\frac{ab}{a+b})^{\frac{1}{3}}[(\frac{b-a}{9ab}) + (\frac{x(a+b)}{3ab})]+ 3(\frac{ab}{a+b})^{\frac{1}{3}}[\frac{2x^2}{9b^2}(1-\frac{1}{9a})]\\
& = & (\frac{ab}{a+b})^{\frac{1}{3}}(\frac{a+b}{ab})[(\frac{b-a}{3(a+b)}) +x + \frac{2x^2}{3b^2}(1-\frac{1}{9a})]\\
& \leq & (\frac{ab}{a+b})^{\frac{1}{3}}(\frac{a+b}{ab})[(\frac{b-a}{3(a+b)})+\frac{2x^2}{3b^2}(1-\frac{1}{9a}) + C +C(\frac{ab}{a+b})^{\frac{1}{2}}]\\
& \leq & (\frac{b-a}{3\sqrt{ab(a+b)}}+\frac{4C^2\sqrt{ab}}{b^2\sqrt{a+b}} +  \frac{C\sqrt{a+b}}{\sqrt{ab}} + C)(\frac{ab}{a+b})^{\frac{5}{6}}(\frac{a+b}{ab}) \\
& \leq & (\frac{1}{3\sqrt{6}}+\frac{1}{6\sqrt{6}} + \frac{1}{2\sqrt{3}}+\frac{1}{2})(\frac{ab}{a+b})^{\frac{5}{6}}(\frac{a+b}{ab}).
\end{eqnarray*}
In above, we used that $C\leq \frac{1}{2}$ and $a,b \geq 6$. Similarly,
\begin{eqnarray*}
denominator(y) & = & [\frac{w_1^2}{b} + \frac{w_2^2}{a}]^{1/2}\\& = & (\frac{ab}{a+b})[\frac{(1+\frac{x}{a})^{\frac{2}{3}}}{b} + \frac{(1-\frac{x}{b})^{\frac{2}{3}}}{a}]^{\frac{1}{2}} \\
& \geq & (\frac{ab}{a+b})^{\frac{1}{3}}[\frac{(1+\frac{2x}{3a}-\frac{x^2}{9a^2})}{b} + \frac{(1-\frac{2x}{3b})}{a}-\frac{x^2}{9a^2}]^{\frac{1}{2}} \\
& = & (\frac{ab}{a+b})^{\frac{1}{3}}[\frac{a(1+\frac{2x}{3a}-\frac{x^2}{9a^2})+b(1-\frac{2x}{3b}-\frac{x^2}{9b^2})}{ab}]^{\frac{1}{2}} \\
& = & (\frac{ab}{a+b})^{\frac{1}{3}}(\frac{a+b}{ab}(1-\frac{x^2}{9ab}))^{\frac{1}{2}} \\
& \geq & (\frac{ab}{a+b})^{\frac{1}{3}}(\frac{a+b}{ab}(1-\frac{4C^2}{9(a+b)}))^{\frac{1}{2}} \\
& \geq & (\frac{ab}{a+b})^{\frac{1}{3}}(\frac{a+b}{ab}(\frac{107}{108}))^{\frac{1}{2}}.
\end{eqnarray*}
Hence we have that $y \leq \frac{\frac{1}{3\sqrt{6}}+\frac{1}{6\sqrt{6}} + \frac{1}{2\sqrt{3}}+\frac{1}{2}}{\sqrt{\frac{107}{108}}} \leq 1$, so that $I_z(a,b) \le \phi(1)+\epsilon$ for $\epsilon \le 0.005$. The lemma statement follows by observing that $1-F_{(a,b)}(z)=1-I_z(a,b) \ge 1-\phi(1)-\epsilon \ge 1-0.845-0.005 \ge 0.15$.  
\end{myproof}
\begin{definition}
\label{def:stocOpt}
For any $X$ and $Y$ real-valued random variables, $X$ is stochastically optimistic for $Y$ if for any $u : R\rightarrow R$ convex and increasing $\Ex[u(X)] \geq \Ex[u(Y )]$.
\end{definition}
\begin{lemma}[Gaussian vs Dirichlet optimism, from \cite{osband2014generalization} Lemma 1]
\label{gvd}
Let $Y = P^TV$ for $V \in [0,1]^S$ fixed and $P \sim Dirichlet(\alpha)$ with $\alpha \in R^S_+$ and $\sum_{i=1}^S \alpha_i \geq 2$. Let $X \sim N(\mu, \sigma^2)$ with $\mu = \frac{\sum_{i=1}^S \alpha_iV_i}{\sum_{i=1}^S \alpha_i}$, $\sigma^2 = (\sum_{i=1}^S \alpha_i)^{-1}$, then $X$ is stochastically optimistic for $Y$.
\end{lemma}
\begin{lemma}[Gaussian vs Beta optimism, \cite{osband2014generalization} Lemma 6]
\label{gvb}
Let $\tilde{Y} \sim Beta(\alpha, \beta)$ for any $\alpha, \beta >0$ and $X \sim N(\frac{\alpha}{\alpha+\beta}, \frac{1}{\alpha + \beta})$. Then $X$ is stochastically optimistic for $\tilde{Y}$ whenever $\alpha +\beta \geq 2$. 
\end{lemma}

\begin{lemma}[Dirichlet vs Beta optimism, \cite{osband2014generalization} Lemma 5]
\label{dvb}
Let $y=p^Tv$ for some random variable $p \sim Dirichlet(\alpha)$ and constants $v \in {\mathcal R}^d$ and $\alpha \in {\mathcal N}^d$. Without loss of generality, assume $v_1 \leq v_2 \leq \cdots \leq v_d$. Let $\tilde \alpha = \sum_{i=1}^d \alpha_i(v_i-v_1)/(v_d-v_1)$ and $\tilde \beta = \sum_{i=1}^d \alpha_i(v_d-v_i)/(v_d-v_1)$. Then, there exists a random variable $\tilde p \sim Beta(\tilde \alpha,\tilde \beta)$ such that, for $\tilde y = \tilde p v_d +(1-\tilde p)v_1$, $\Ex[\tilde y | y] =\Ex[y]$.
\end{lemma}

\begin{lemma}
\label{lem:negso}
If $\Ex[X] = \Ex[Y]$ and $X$ is stochastically optimistic for $Y$, then $-X$ is stochastically optimistic for $-Y$.
\end{lemma}
\begin{myproof}
By Lemma 3.3 in \cite{osband2014generalization}, $X$ stochastically optimistic for $Y$ is equivalent to having $X =_D Y+A + W$ with $A\geq 0$ and  $\Ex[W | Y+A] = 0$ for all values $y+a$. Taking expectation of both sides, we get that $\Ex[X] = \Ex[Y] +\Ex[A] + \Ex[W]$ and since $\Ex[X]=\Ex[Y]=0$ and $\Ex[W] = \Ex[\Ex[W|Y+A]] = 0$ we get that $\Ex[A] =0$. Since $A\geq0$, $A=0$. Also note that $\Ex[W | Y=y]= 0$ for all $y$.

Now we can show that $-X$ is stochastically optimistic for $-Y$ as follows: From above, $-X =_D -(Y +A + W) = -Y+(-W)$. Then for all $y'$, $\Ex[-W | -Y=y'] = -\Ex[W|Y=-y'] = 0$ by definition of $W$. Therefore, $-X$ is stochastically optimistic for $-Y$.
\end{myproof}

\begin{corollary}
\label{cor:soconc1} 
Let $Y$ be any distribution with mean $\mu$ such that $X \sim N(\mu, \sigma^2)$ is stochastically optimistic for $Y$. Then with probability $1-\rho$, \[|Y-\mu| \leq \sqrt{2\sigma^2\log(2/\rho)}.\]
\end{corollary}
\begin{myproof}
For any $s>0$, and $t$, and applying Markov's inequality,
\[P(Y - \mu > t) = P(Y > \mu + t) = P(e^{sY} > e^{s(\mu +t)}) \leq \frac{\Ex[e^{sY}]}{e^{s(\mu +t)}}.\]
By Definition \ref{def:stocOpt}, taking $u(a) = e^{sa}$, which is a convex and increasing function, $\Ex[e^{sY}] \leq \Ex[e^{sX}]$, and hence \[P(Y - \mu > t) \leq  \frac{\Ex[e^{sX}]}{e^{s(\mu +t)}} = \frac{e^{\mu s + \frac{1}{2}\sigma^2s^2}}{e^{s(\mu + t)}} = e^{\frac{1}{2}\sigma^2s^2 - st}.\]

Since the above holds for all $s>0$, using $s=\frac{t}{\sigma^2}$, $P(Y - \mu > t) \leq e^{-\frac{t^2}{2\sigma^2}}$. 

Similarly, for the lower tail bound, we have for any $s>0$, \[P(Y - \mu < -t) = P(-Y > -\mu + t) = P(e^{s(-Y)} > e^{s(-\mu +t)}) \leq \frac{\Ex[e^{s(-Y)}]}{e^{s(-\mu +t)}}.\]
By Lemma \ref{lem:negso}, $-X$ is stochastically optimistic for $-Y$, so $\Ex[e^{s(-Y)}] \leq \Ex[e^{s(-X)}]$, and hence \[P(Y - \mu < -t) \leq  \frac{\Ex[e^{s(-X)}]}{e^{s(-\mu +t)}} = \frac{e^{-\mu s + \frac{1}{2}\sigma^2s^2}}{e^{s(-\mu + t)}} = e^{\frac{1}{2}\sigma^2s^2 - st}.\]

Again letting $s=\frac{t}{\sigma^2}$, $P(Y - \mu < -t) \leq e^{-\frac{t^2}{2\sigma^2}}$. 

Then, for $t=\sqrt{2\sigma^2 \log(2/\rho)}$, we have that
\[P(|Y - \mu| \leq  \sqrt{2\sigma^2 \log(2/\rho)}) \geq 1-\rho.\]
\end{myproof}

\comment{
\begin{lemma}[Binomial, Multinomial] 
\label{lem:BB}
Let $\hat{Y}=\hat{p}^Tv$ where $\hat{p} \in \Delta^S$ be distributed as multinomial average with parameter $n, p$ and fixed $v\in R^d$, where $0\le v_i\le D$. Then, there exists a random variable distributed as $\hat{q} \sim \frac{1}{n} Binomial(n, \frac{p^Th}{D})$ such that, $\Ex[\hat q | \hat Y] = \frac{1}{D}\hat Y$.
\end{lemma}
\begin{myproof}
Let $X^j_i, j=1,\ldots, n$ denote the outcomes of the trials used to define $\hat{p}_i$, that is,
\[\hat{p}_i := \sum_{j=1}^n X^j_i/n\]
where $X^j_i, j=1,\ldots, n$ are distributed as $X^j_i \sim Multivariate(p,1)$.

For every $i$, define $n$ i.i.d. variables $Y^{j}_i, j=1,\ldots, n$, where $Y_i^j \sim Bernoulli(v_i/D)$, and is independent of $X^j_i$. Define $\hat{q}$ as:
\[ \hat{q} =\frac{1}{n} \sum_i \sum_{j=1}^n X^j_i Y^j_i/n \]
Let ${\cal X}=\{X_{i,j}, i=1,\ldots, S,j=1,\ldots, n\}$. Then,
\begin{eqnarray*}
 \Ex[\hat{q} | \hat{p}^Tv, n] & = & \Ex[\Ex[\hat{q} | {\cal X}, \hat{p}^Tv, n] | \hat{p}^Tv, n] \\
& = & \Ex[\Ex[\hat{q} | {\cal X}, n] | \hat{p}^Tv, n] \\
& = & \frac{1}{n}\Ex[\Ex[\sum_{i,j}  X^j_i Y^j_i |  {\cal X}, n]  | \hat{p}^Tv, n]\\
& = & \frac{1}{n}\Ex[ \sum_{i,j} X^j_i \Ex[Y^j_i]  | \hat{p}^Tv, n]\\
& = & \frac{1}{n}\Ex[\sum_{i,j} X^j_i \frac{v_i}{D}   | \hat{p}^Tv, n]\\
& = & \hat{p}^Tv/D.
\end{eqnarray*}

Also, $n\hat{q}$ is a binomial random variable $Binomial(n, \frac{1}{D} p^Tv)$ since it is formed by sum of  outcomes of $n$ trials $\sum_{j=1}^n Z^j$, where each trail $Z^j = \sum_i X^j_i Y^j_i$ is an independent Bernoulli trial: takes value $1$ with probability $\sum_i p_i v_i/D$.
\end{myproof}

\comment{Further, for any $b$, 
$\Pr(\max_v \hat{p}^Tv - p^Tv \ge b ) = \int_{v^*} \Pr(\hat{p}^Tv^* - p^Tv^* \ge b | \arg \max_v \hat{p}^Tv  - p^Tv  =v^*) \Pr(\arg \max_v \hat{p}^Tv  - p^Tv =v^*) dv^*
= \int_{v^*} \Pr(\hat{p}^Tv^*  - p^Tv^* \ge b | {\cal X}: {\cal E}) \Pr({\cal X}: {\cal E})
= \int_{v^*, \mu=p^Tv^*} \Pr(\hat q(\mu) -\mu \ge b | {\cal X}: {\cal E}) \Pr({\cal X}: {\cal E}) dv^*
= \int_{\mu} \int_{v^*: p^Tv^*=\mu} \Pr(\hat q(\mu) \ge b | {\cal X}: {\cal E}) \Pr({\cal X}: {\cal E}) dv^* d\mu
= \int_{\mu} \int_{v^*: p^Tv^*=\mu} \Pr(\hat q(\mu) \ge b , {\cal X}: {\cal E}(v^*)) dv^* d\mu
\le \int_{\mu} \Pr(\hat q(\mu) \ge \mu + \epsilon)) d\mu
\le \int_{0}^1 \exp{(-n\epsilon^2)} d\mu
 \le \int_{0}^1 \rho d\mu
= \rho
$
}
\begin{corollary}
\label{cor:BB}
For $X=D\hat{q}$, $Y=\hat{p}^Tv$ (with $\hat q$ and $\hat p^Tv$ as defined in the previous lemma), $X$ is stochastically optimistic for $Y$. 
\end{corollary}
\begin{myproof}
We have
\[\Ex[X-Y | Y] = \Ex[D\hat{q} - \hat{p}^Tv | \hat{p}^Tv] = 0.\]
Then stochastic optimism follows from applying the optimism equivalence condition from Lemma 3 (Condition 3) of \cite{osband2014generalization}.
\end{myproof}

\comment{
\begin{corollary}
\label{cor:soconc2}
Let $Y$ be any distribution with mean $\mu\in [0,1]$ such that $X \sim \frac{1}{n} \text{Binomial}(n,\mu)$ is stochastically optimistic for $Y$. Then,
with probability $1-\rho$, 
\[Y-\mu \leq \sqrt{3\frac{\ln(1/\rho)}{n}}.\]
\end{corollary}
\begin{myproof}
For any $s>0$, and $\delta>0$, and applying Markov's inequality,
\[P(Y > (1+\delta)\mu ) = P(e^{sY} > e^{s(1+\delta)\mu}) \leq \frac{\Ex[e^{sY}]}{e^{s(1+\delta)\mu}}.\]
By Lemma \ref{lem:BB}, $X\succeq_{so} Y$. Using Definition \ref{def:stocOpt}, taking $u(a) = e^{sa}$, which is a convex and increasing function, $\Ex[e^{sY}] \leq \Ex[e^{sX}]$, and hence 
\[P(nY > (1+\delta)n\mu ) \leq  \frac{\Ex[e^{snX}]}{e^{s(1+\delta)n\mu}} \le \left(\frac{e^{\delta}}{(1+\delta)^{(1+\delta)}}\right)^{n\mu} \le e^{-\frac{\delta^2 n\mu}{2+\delta}} \]
using the moment generating function of Binomial distribution and $s=\ln(1+\delta)$.


Now, using $\delta = \frac{1}{\mu} \sqrt{\frac{3\ln(1/\rho)}{n}}$, for $n\ge 3\ln(1/\rho)$, $\delta \le 1/\mu$  so that $2+\delta \le \frac{3}{\mu}$, and we get
\[P(nY > (1+\delta)n\mu ) \leq e^{-\frac{1}{3} n\delta^2\mu^2} \le \rho\]
so that
\[Y-\mu \leq \sqrt{\frac{3\ln(1/\rho)}{n}}.\]
with probability $1-\rho$.

\end{myproof}
}

}  

\end{document}